\documentclass{article}

\usepackage{microtype}
\usepackage{graphicx}
\usepackage{subcaption}
\usepackage{booktabs} 

\usepackage{hyperref}

\usepackage[preprint]{icml2026}

\usepackage{amsmath}
\usepackage{amssymb}
\usepackage{mathtools}
\usepackage{amsthm}

\usepackage[capitalize,noabbrev]{cleveref}

\theoremstyle{plain}
\newtheorem{theorem}{Theorem}[section]
\newtheorem{proposition}[theorem]{Proposition}

\theoremstyle{definition}

\theoremstyle{remark}

\usepackage[textsize=tiny]{todonotes}
\usepackage{enumitem}
\usepackage{inconsolata}

\usepackage{graphicx}

\usepackage{hyperref}       
\usepackage{url}            
\usepackage{booktabs}       
\usepackage{amsfonts}       
\usepackage{nicefrac}       
\usepackage{microtype}      
\usepackage{xcolor}         
\usepackage{amsmath}
\usepackage{amssymb}
\usepackage{mathtools}
\usepackage{amsthm}
\usepackage{multirow}
\usepackage{graphicx}
\usepackage{colortbl}
\usepackage{arydshln}
\usepackage{enumitem}
\usepackage{wrapfig}
\usepackage{float}
\usepackage{adjustbox}

\usepackage[capitalize,noabbrev]{cleveref}
\usepackage{bbding}
\usepackage{pifont}
\usepackage{minitoc}
\usepackage{booktabs}
\usepackage{minitoc}

\definecolor{wjy}{named}{red}
\definecolor{myblue}{RGB}{224,247,250}
\definecolor{myblue2}{RGB}{186,230,251}

\definecolor{backblue}{RGB}{210, 230, 250}

\definecolor{backred}{RGB}{255, 223, 223}

\definecolor{backgreen}{RGB}{220,244,229}

\definecolor{back_deepblue}{RGB}{180, 210, 240}

\definecolor{back_deepred}{RGB}{255, 200, 200}

\definecolor{back_deepgreen}{RGB}{190, 230, 210}

\definecolor{mygray}{gray}{0.95}

\definecolor{greentable3}{rgb}{0,0.5,0}

\definecolor{mgreen}{RGB}{6,128,67}
\definecolor{mgray}{RGB}{128,128,128}
\definecolor{mygreen}{RGB}{233,247,234}

\usepackage[most,skins,theorems]{tcolorbox}
\tcbset{
  takeawaysbox/.style={
    title=Takeaways,
    colback=default!80,
    colframe=black,
    fonttitle=\bfseries\small,
    coltitle=white,
    colbacktitle=black,
    enhanced,
    attach boxed title to top left={xshift=2.5mm,yshift=-2.5mm},
    boxed title style={rounded corners, size=small, colframe=black, colback=black},
    width=\linewidth,
    arc=3.5mm
  }
}
\definecolor{cvprblue}{rgb}{0.21,0.49,0.74}
\definecolor{lightblue}{RGB}{220,235,250}  
\definecolor{lightgreen}{RGB}{230,242,222} 
\definecolor{lightgreen2}{RGB}{241,248,235}
\definecolor{lightgreen3}{RGB}{236,245,228}
\definecolor{lightyellow}{RGB}{255,255,210} 
\definecolor{lightorange}{RGB}{255,235,210} 
\definecolor{default}{RGB}{255,255,255}
\definecolor{pseudoblue}{RGB}{55,113,175}
\definecolor{mgreen}{RGB}{6,128,67}
\definecolor{mgray}{RGB}{128,128,128}
\definecolor{mygreen}{RGB}{233,247,234}
\definecolor{mygray}{gray}{0.97}
\definecolor{mgray}{RGB}{240,240,240}
\definecolor{purple}{RGB}{128,128,247}
\definecolor{pink}{RGB}{252,237,236}
\definecolor{pink2}{RGB}{234,223,222}
\definecolor{pink3}{RGB}{237,227,228}
\definecolor{blue1}{RGB}{229,243,247}
\usepackage{pifont}

\usepackage{tikz}

\NewDocumentCommand{\circnum}{O{black!25} m}{%
  \tikz[baseline=(char.base)]{
    \node[shape=circle, fill=#1, inner sep=0.8pt, font=\small\bfseries] (char)
      {\textcolor{white}{#2}};
  }%
}
\NewDocumentCommand{\circnumgreen}{O{teal!18} m}{%
  \tikz[baseline=(char.base)]{
    \node[shape=circle, fill=#1, inner sep=1pt, font=\small\bfseries] (char)
      {\textcolor{teal!70!black}{#2}};
  }%
}
\NewDocumentCommand{\circnumpink}{O{orange!20!white} m}{%
  \tikz[baseline=(char.base)]{
    \node[shape=circle, fill=#1, inner sep=1pt, font=\small\bfseries] (char)
      {\textcolor{orange!80!red}{#2}};
  }%
}
\NewDocumentCommand{\circnumblue}{O{cyan!25} m}{%
  \tikz[baseline=(char.base)]{
    \node[shape=circle, fill=#1, inner sep=1pt, font=\small\bfseries] (char)
      {\textcolor{blue!70!black}{#2}};
  }%
}
\NewDocumentCommand{\circnumpurple}{O{violet!20} m}{%
  \tikz[baseline=(char.base)]{
    \node[shape=circle, fill=#1, inner sep=1pt, font=\small\bfseries] (char)
      {\textcolor{violet!80!black}{#2}};
  }%
}

\icmltitlerunning{SSL: Sweet Spot Learning for Differentiated Guidance in Agentic Optimization}

\begin{document}

\twocolumn[
  \icmltitle{SSL: Sweet Spot Learning for Differentiated Guidance in Agentic Optimization}



  \icmlsetsymbol{equal}{*}
  \icmlsetsymbol{co}{\dag}

  \begin{icmlauthorlist}
    \icmlauthor{Jinyang Wu}{1,equal}
    \icmlauthor{Changpeng Yang}{2,equal}
    \icmlauthor{Yuhao Shen}{3}
    \icmlauthor{Fangzhi Xu}{4}
    \icmlauthor{Bolin Ni}{5}
    \icmlauthor{Chonghua Liao}{1}
    \icmlauthor{Yuchen Liu}{2}
    \icmlauthor{Hongzhen Wang}{2}
    \icmlauthor{Shuai Nie}{2}
    \icmlauthor{Shuai Zhang}{1}
    \icmlauthor{Haoran Luo}{4,co}
    \icmlauthor{Jiaming Xu}{2,co}
  \end{icmlauthorlist}

  \icmlaffiliation{1}{Tsinghua University}
  \icmlaffiliation{2}{Xiaomi Corporation}
  \icmlaffiliation{3}{Zhejiang University}
  \icmlaffiliation{4}{Nanyang Technological University}
  \icmlaffiliation{5}{Institute of Automation, Chinese Academy of Sciences}
  \icmlcorrespondingauthor{Jinyang Wu}{wu-jy23@mails.tsinghua.edu.cn}
  \icmlcorrespondingauthor{Changpeng Yang}{yangchangpeng@xiaomi.com}

  \icmlkeywords{Machine Learning, ICML}

  \vskip 0.3in
]



\printAffiliationsAndNotice{}  

\begin{abstract}
Reinforcement learning with verifiable rewards has emerged as a powerful paradigm for training intelligent agents. However, existing methods typically employ binary rewards that fail to capture quality differences among trajectories achieving identical outcomes, thereby overlooking potential diversity within the solution space. Inspired by the ``sweet spot'' concept in tennis-the racket's core region that produces optimal hitting effects, we introduce \textbf{S}weet \textbf{S}pot \textbf{L}earning (\textbf{SSL}), a novel framework that provides differentiated guidance for agent optimization. SSL follows a simple yet effective principle: progressively amplified, tiered rewards guide policies toward the sweet-spot region of the solution space. This principle naturally adapts across diverse tasks: visual perception tasks leverage distance-tiered modeling to reward proximity, while complex reasoning tasks reward incremental progress toward promising solutions. We theoretically demonstrate that SSL preserves optimal solution ordering and enhances the gradient signal-to-noise ratio, thereby fostering more directed optimization. Extensive experiments across GUI perception, short/long-term planning, and complex reasoning tasks show consistent improvements over strong baselines on 12 benchmarks, achieving up to 2.5× sample efficiency gains and effective cross-task transferability. Our work establishes SSL as a general principle for training capable and robust agents.
\end{abstract}

\begin{figure}[htp!]
\vskip -0.08in
\includegraphics[width=0.96\linewidth]{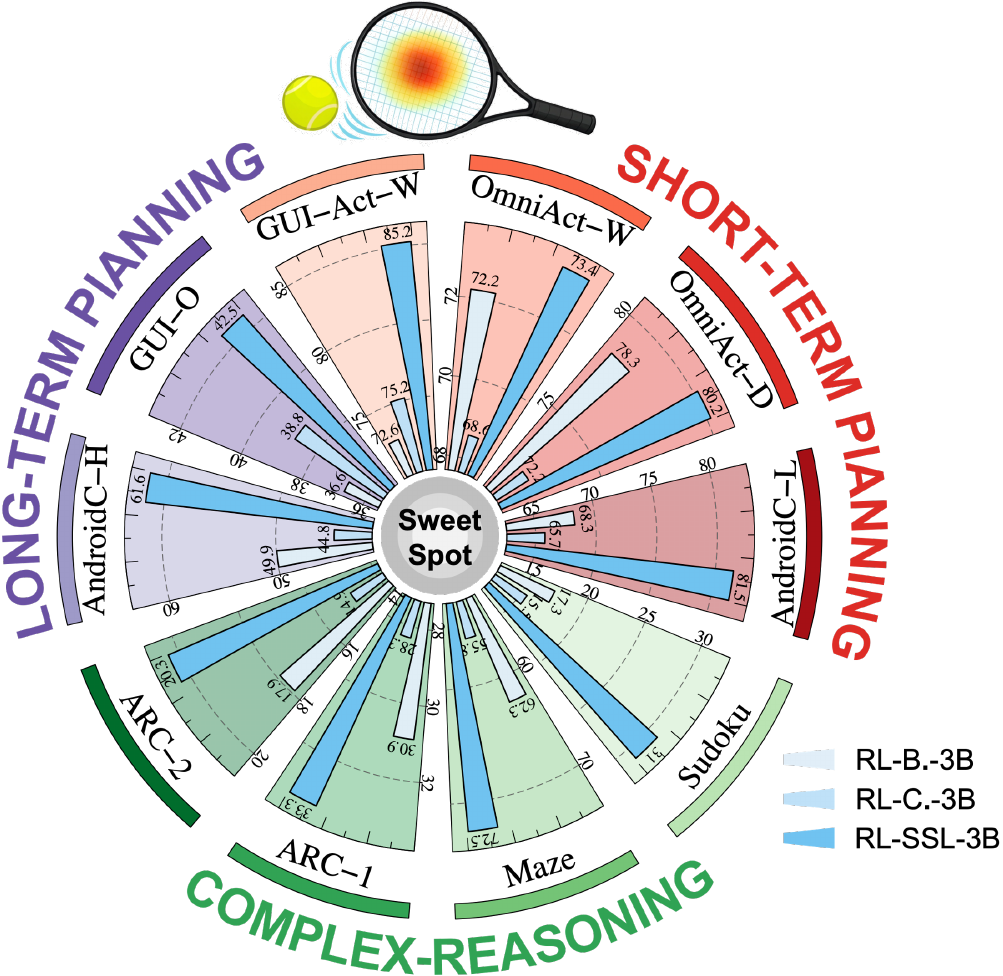}
\caption{\textbf{Performance comparison with Binary (RL-B.) and Continuous (RL-C.) Reward RL.} SSL (our method) exhibits remarkable improvements across diverse tasks, including long-/short-term planning and complex reasoning.}
\label{head_radar}
\vskip -0.15 in
\end{figure}

\section{Introduction}
\label{sec:intro}
Reinforcement learning with verifiable rewards (RLVR) has emerged as a transformative paradigm for training intelligent multimodal agents with sophisticated reasoning and planning capabilities~\cite{jaech2024openai,guo2025deepseek,team2025kimi,nguyen-etal-2025-gui,jiang-etal-2025-towards}. Unlike supervised fine-tuning (SFT) that typically depends on costly human-annotated reasoning chains, RLVR directly leverages automatically computable reward signals to optimize agent behaviors. This approach enables agents to autonomously develop complex cognitive abilities, including chain-of-thought reasoning~\cite{wei2022chain}, problem decomposition~\cite{zhou2022least}, self-correction~\cite{gandhi2025cognitive}, and multi-step planning~\cite{luo2025gui}.

However, mainstream RLVR typically adopts binary rewards, grouping all trajectories into success or failure and obscuring finer distinctions among them. Consider two GUI navigation trajectories that both successfully open application settings: one reaches the target in three actions, while another succeeds after eight redundant steps—yet both receive the same reward. Similarly, in maze navigation, multiple paths may reach the goal with vastly different characteristics: some take circuitous detours, while others follow direct corridors. These examples illustrate how coarse, undifferentiated rewards induces three challenges: \textbf{(i) Optimization Ambiguity}: without differentiated guidance, gradient updates may lack directional information to indicate which task-advancing behaviors merit reinforcement or discouragement~\cite{ng1999policy,baheri2023understanding};  \textbf{(ii) Learning Inefficiency}: given reward signals fail to reveal solution-quality differences, agents cannot effectively extract fruitful knowledge from diverse trajectories, leading to under-exploration of solution space and poor sample utilization~\cite{eschmann2021reward}; \textbf{(iii) Policy Fragility}: policies may overfit to incidental patterns (\emph{e.g.}, fortunate click) rather than robust task understanding~\cite{velu2023hindsight}.

\textit{Can we design a unified reward principle that provides differentiated guidance across the full solution space?} Inspired by the ``sweet spot'' in tennis~\cite{cross1998sweet}—the racket's core region that produces optimal hitting impact, we propose \textbf{S}weet \textbf{S}pot \textbf{L}earning (\textbf{SSL}), a versatile reward modeling scheme for training intelligent agents. The key insight is simple yet powerful: progressively amplified, tiered rewards guide policies toward the sweet-spot region of solution space. Rather than relying on a coarse success-failure split, SSL assigns hierarchical rewards to trajectories based on trajectory proximity to task completion. As shown in Figure~\ref{schema}, this principle adapts naturally across tasks: for GUI perception and navigation, progress manifests as spatial proximity to targets, quantified via distance-tiered rewards; for complex reasoning (e.g., maze navigation), rewards reflect incremental alignment with ground-truth solutions, inducing a progress-tiered structure. Through quality-conditioned reward shaping, SSL provides differentiated guidance that steer exploration toward high-quality solutions.

\begin{figure}[t!]
\vskip -0.05in
\centering
\includegraphics[width=0.96\linewidth]{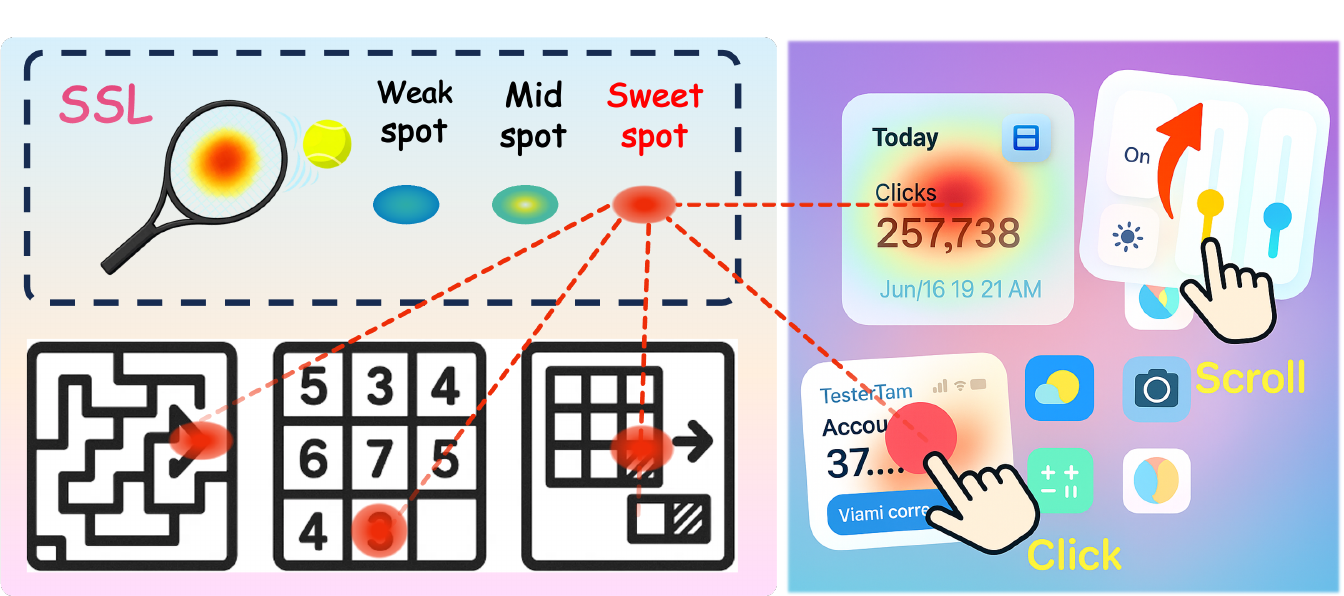}
\caption{\textbf{Sweet Spot Learning Overview.} We visualize sweet-spot zones and their instantiation across diverse tasks .}
\label{schema}
\vskip -0.15 in
\end{figure}

\textit{\underline{Theoretically}}, we demonstrate that SSL preserves the relative quality ordering over the solution space and improves gradient signal-to-noise ratios by providing more informative feedback for policy updates. Unlike sparse binary rewards that offer only success-failure supervision, SSL delivers granular feedback that accelerates convergence and improves sample efficiency. \textit{\underline{Empirically}}, as shown in Figure~\ref{head_radar}, we validate SSL's generality and effectiveness across GUI perception, short/long-term planning, and complex reasoning, demonstrating consistent improvements over representative baselines. SSL offers three key advantages: \textbf{(i) Clear Optimization Guidance}: tiered, proximity-aligned reward signals provide stronger directional feedback during policy updates, enabling more stable and effective optimization; \textbf{(ii) Enhanced Efficiency}: SSL matches or surpasses GRPO using only 40\% of its samples, achieving up to 2.5$\times$ data-efficiency; \textbf{(iii) Transferability \& Superior Performance:}
SSL consistently improves performance across diverse tasks, demonstrating strong transferability. Our contributions are:

\begin{itemize}[leftmargin=*, itemsep=0pt, topsep=4pt]
   \item We propose \emph{sweet spot learning}, a unified reward principle that integrates sweet-spot modeling into reinforcement learning. SSL offers tiered, quality-sensitive rewards with theoretical guarantees on preserving optimal solution ordering and improving gradient optimization.

   \vspace{0.03in}

   \item We instantiate SSL across diverse agent tasks: distance-tiered rewards for visual perception and navigation, and progress-tiered rewards for complex reasoning. This demonstrates SSL's broad applicability.
   
   \item We conduct extensive experiments on 12 benchmarks spanning GUI perception, short-/long-term planning, and complex reasoning, showing consistent gains over strong baselines, enhanced efficiency, and transferability.
\end{itemize}
\section{Related Work}
\label{sec:related}
\paragraph{RLVR for Agent Optimization.}
RLVR has become a key paradigm for training capable agents by leveraging automatically computable success criteria~\cite{guo2025deepseek,team2025kimi,wu2025thought}. Recent work typically applies on-policy algorithms like GRPO with binary terminal rewards across tasks such as math reasoning~\cite{wang2025survey}, GUI navigation~\cite{nguyen-etal-2025-gui,zhang2024large}, and puzzle solving~\cite{sudoku-extreme}. However, binary rewards treat all successful trajectories equally, ignoring solution quality-\emph{e.g.}, a three-step GUI path and an eight-step detour receive the same reward. This leads to sample inefficiency and policy fragility (overfitting to incidental patterns)~\cite{sutton1998reinforcement,xu2025towards,amodei2016concrete}. SSL addresses these issues by providing differentiated guidance through quality-ordered zone partitioning.

\begin{figure*}[htp!]
\centering
\includegraphics[width=0.99\linewidth]{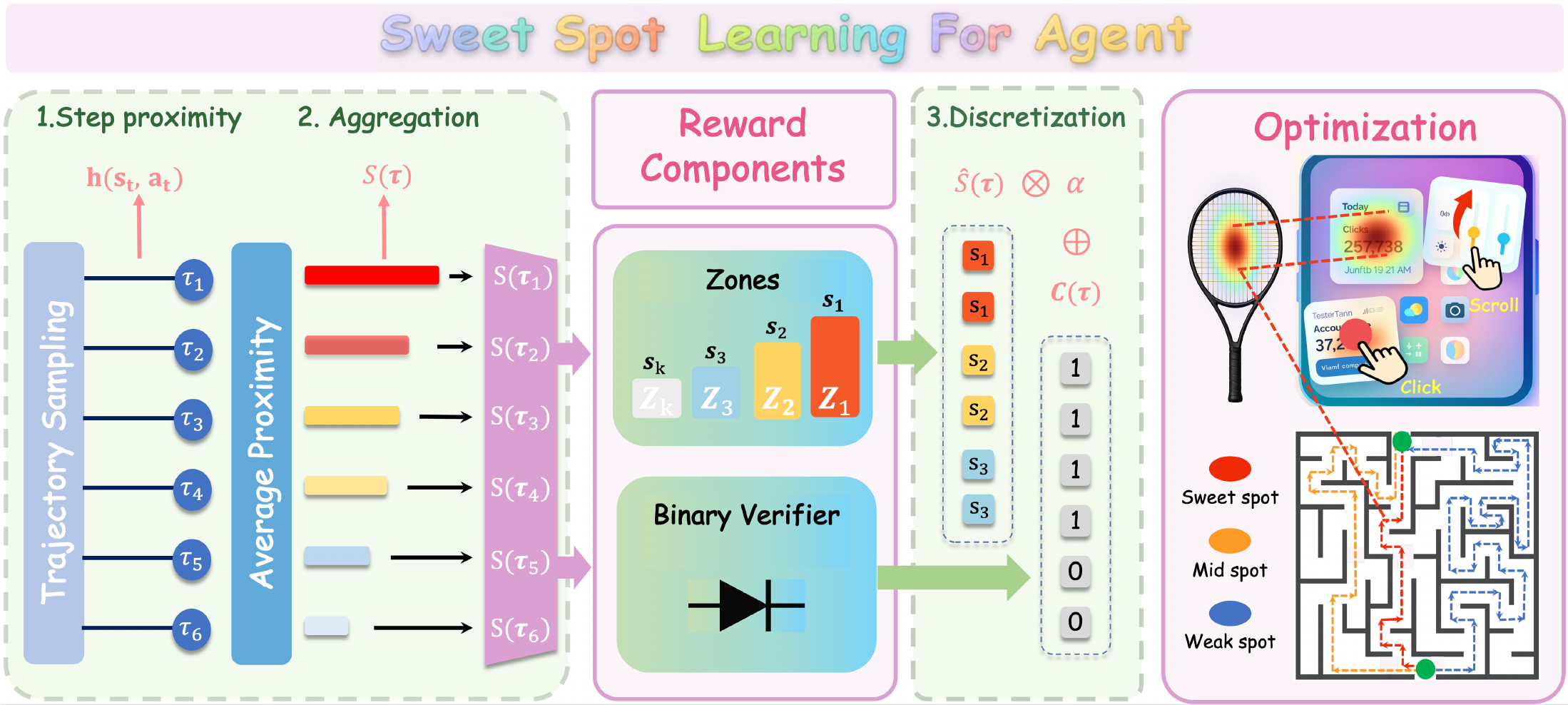}
\caption{\textbf{Flowchart of SSL.} SSL computes step proximities, aggregates them into trajectory-level scores, discretizes them into predefined sweet-spot zones, and finally applies the resulting structured rewards to downstream optimization tasks.}
\label{framework}
\vskip -0.1 in
\end{figure*}

\vskip -0.15 in
\paragraph{Reward Shaping in RLVR.}
Reward shaping techniques augment sparse terminal signals with detailed feedback~\cite{ng1999policy,grzes2017reward,hu2020learning}. Classical potential-based methods rely on handcrafted Markovian potentials~\cite{sutton1998reinforcement}, while modern approaches employ distance rewards for navigation~\cite{ivanitskiy2023configurable}, milestone rewards for reasoning~\cite{luo2025ursa}, and hindsight relabeling for sparse rewards~\cite{velu2023hindsight}. Yet, these methods often require considerable task-specific designs or lack theoretical guarantees~\cite{ibrahim2024comprehensive,zou2019reward}. In contrast, SSL introduces a unified, task-agnostic principle: partitioning solution space into quality-ordered zones and assigning tiered rewards based on proximity to optimality. This enables generalizable, efficient learning with preservation of optimal solution ordering.

\section{Methodology}
\label{sec:method}

\subsection{Problem Formulation}
\label{subsec:formulation}
We formulate agent training as a Markov Decision Process $\mathcal{M} = (\mathcal{S}, \mathcal{A}, \mathcal{T}, R)$, where $\mathcal{S}$ denotes the state space, $\mathcal{A}$ represents the action space, $\mathcal{T}: \mathcal{S} \times \mathcal{A} \rightarrow \Delta(\mathcal{S})$ defines the state transition function, $R$ is the reward function. A trajectory $\tau = (s_0, a_0, \ldots, s_T, a_T)$ is generated by policy $\pi_\theta(a|s)$, and a task-specific verifier $C(\tau) \in \{0, 1\}$ evaluates terminal success (\emph{e.g.}, whether goals are achieved for planning tasks). Standard RLVR maximizes expected return under binary rewards:
\begin{equation}
R_{\text{bin}}(\tau) = C(\tau), \qquad J_{\text{bin}}(\theta) = \mathbb{E}_{\tau \sim \pi_\theta}[R_{\text{bin}}(\tau)].
\end{equation}

\subsection{Sweet Spot Learning}
\label{subsec:framework}
\textit{Core Principle: Actions closer to the task's sweet-spot deserve higher rewards, but tiered.}
Inspired by the tennis ``sweet spot''-where strikes nearer to the racket’s optimal zone yield stronger effects, we propose Sweet Spot Learning, which assigns progressively higher rewards to trajectories that approach task-optimal regions. The key insight is to partition the solution space into hierarchical zones based on proximity to optimality, delivering directional guidance toward high-quality solutions. We present SSL framework in Figure~\ref{framework} and pseudo-code in Algorithm~\ref{alg:ssl}.

\vspace{-0.12in}
\paragraph{Sweet Spot Learning Reward.}
We first introduce a \emph{discretized sweet-spot value} $\hat{S}(\tau)\!\in[0,1]$, where the hat means discretization. This value measures how closely a trajectory approaches high-quality completion. Combining this with binary verification yields the SSL reward:
\begin{equation}
\label{eq:ssl-general}
R_{\text{SSL}}(\tau) = C(\tau) + \alpha \, \hat{S}(\tau), \qquad \alpha > 0,
\end{equation}
where $C(\tau)\in\{0,1\}$ encodes trajectory correctness, $\hat{S}(\tau)$ guides learning toward superior solution regions, and $\alpha$ controls the strength of this guidance.

\vspace{-0.15in}
\paragraph{Predefined Components.}
SSL relies on two components that form the basis for computing trajectory sweet-spot values, which are instantiated
concretely in Section~\ref{subsec:instantiations}.

\textit{\underline{Sweet-spot zones.}} 
SSL begins by defining a finite set of ordered sweet-spot zones $\{\mathcal{Z}_k, s_k\}_{k=1}^K$, where each interval $\mathcal{Z}_k$ specifies a proximity range and $s_k$ is the sweet-spot value assigned to the region. Concretely, we parameterize zones by boundaries $0 = b_K < \cdots < b_1 < b_0 = 1$ and set $\mathcal{Z}_k = [b_{k-1}, b_k)$.

\textit{\underline{Step proximity.}} 
SSL decomposes a trajectory into step-level units for fine-grained assessment of action contributions. For each step $(s_t,a_t)$, SSL computes a proximity score $h(s_t,a_t)\in[0,1]$, which measures how well the action taken at state $s_t$ aligns with desirable task behavior.

\vspace{-0.1in}
\paragraph{Trajectory-Level Aggregation.}
Having established step-level proximity, SSL aggregates these local evaluations into a unified trajectory-level proximity measure:
\begin{equation}
    S(\tau)=\frac{1}{T+1}\sum_{t=0}^{T} h(s_t,a_t).\label{eq:solution_quality}
\end{equation}
This aggregation captures how well the entire trajectory progresses toward desirable solution regions. By bridging local and global granularities, SSL reveals gradual improvement trends that single-step signals cannot capture. This trajectory-level proximity then serves as input to subsequent discretization step that produces tiered sweet-spot values.

\vspace{-0.1in}
\paragraph{Trajectory-Proximity Score Discretization.}
The trajectory proximity $S(\tau)$ is then mapped into a tiered sweet-spot value via predefined zone-based discretization. Using the boundary $0 = b_K < \cdots < b_1 < b_0 = 1$ that define $\mathcal{Z}_k = [b_{k-1}, b_k)$, we obtain the discretized sweet-spot value
\begin{equation}
\widehat{S}(\tau)
= \sum_{k=1}^{K} s_k \,\mathbf{1}\!\left[b_{k-1} \le S(\tau) < b_k\right].
\end{equation}
While the mapping itself is simple, its effectiveness stems from the fact
that the predefined zones $(\mathcal{Z}_k)$ are constructed to reflect
task-level structure. Mapping $S(\tau)$ into these zones suppresses small and
noisy score fluctuations, yielding more discriminative reward tiers that
provide clearer learning signals for agent optimization.

\vskip -0.1in
\paragraph{Policy Optimization.}
Given SSL rewards $\{R_{\text{SSL}}(\tau_i)\}_{i=1}^N$ for sampled trajectories, we update the policy model following the standard RLVR paradigm. SSL integrates seamlessly with mainstream methods such as GRPO~\cite{guo2025deepseek} by simply replacing binary rewards with tiered SSL rewards. This differentiated feedback yields more stable gradient estimates and higher sample efficiency than binary signals.

\subsection{Agent Instantiations under SSL}
\label{subsec:instantiations}
We show SSL principle's flexibility by applying it to two broad agent classes: GUI agents that reason over visual screens to select and execute grounded actions, and complex reasoning agents that operate on structured grids (mazes, Sudoku, ARC patterns). Despite substantial differences in semantics and action spaces, SSL adapts naturally through task-aligned zone design and proximity scoring.

\vspace{-0.05in}
\paragraph{\circnumgreen{1} GUI Agents:}
\textbf{\textit{Distance-Tiered Zones.}}
In visual grounding, the agent outputs a 2D point $\mathbf{p} = (x, y)$ for a target UI element. Given ground-truth bounding box $B = (x_1, y_1, x_2, y_2)$ and gaussian spatial fields are commonly used to model human pointing behavior and motor uncertainty~\cite{mackenzie1992fitts}, we compute continuous proximity $\phi(\mathbf{p}; B) \in (0,1]$ via a Gaussian field centered at $B$ using a normalized Mahalanobis distance. Points outside $B$ receive $\phi(\mathbf{p}; B) = 0$, treating them as incorrect. This aligns with the common-correctness setting, where correctness is prioritized before sweet-spot refinement.
We parameterize the Gaussian width $\sigma$ proportional to the bounding box to
ensure that proximity judgments adapt to target scale. 

\begin{algorithm}[ht!]
\caption{Sweet Spot Learning (SSL)}
\label{alg:ssl}
\begin{algorithmic}[1]
\REQUIRE Policy $\pi_\theta$; proximity function $h(s,a) \in [0,1]$; sweet-spot zones $\{\mathcal{Z}_k, s_k\}_{k=1}^K$; coefficient $\alpha > 0$
\FOR{each training iteration}
  \STATE Sample trajectory set $\{\tau_i\}_{i=1}^N \sim \pi_\theta$
  \FOR{each trajectory $\tau_i = \{(s_t, a_t)\}_{t=0}^{T_i}$}
    \STATE $C(\tau_i) \leftarrow \text{Verifier}(\tau_i) \in \{0, 1\}$
    
    \STATE \textcolor{pseudoblue}{// Trajectory-level aggregation}
    \STATE $h_t \leftarrow h(s_t, a_t)$ for $t = 0, \ldots, T_i$
    \STATE $S(\tau_i) \leftarrow \frac{1}{T_i + 1} \sum_{t=0}^{T_i} h_t$
    
    \STATE \textcolor{pseudoblue}{// Trajectory-Proximity Score Discretization}
    \STATE Find zone index $k$ such that $S(\tau_i) \in \mathcal{Z}_k$
    \STATE $\widehat{S}(\tau_i) \leftarrow s_k$
    
    \STATE \textcolor{pseudoblue}{// SSL reward computation}
    \STATE $R_{\text{SSL}}(\tau_i) \leftarrow C(\tau_i) + \alpha \cdot \widehat{S}(\tau_i)$
  \ENDFOR
  
  \STATE Update $\theta$ using policy gradient with $\{R_{\text{SSL}}(\tau_i)\}_{i=1}^N$
\ENDFOR
\end{algorithmic}
\end{algorithm}

We then instantiate sweet-spot tiers $\{\mathcal{Z}_k, s_k\}_{k=1}^K$ by selecting thresholds $0 = b_K < \cdots < b_1 < b_0 = 1$.
Each $b_k$ corresponds to the boundary value of a $\sigma$-level band of the Gaussian field inscribed within $B$, discretizing the continuous decay into $K$
nested sweet-spot regions. A trajectory falls into tier $\mathcal{Z}_k$ if
\begin{equation}
\phi(\mathbf{p}; B) \in (b_k, b_{k-1}], \qquad \widehat{S}(\tau) = s_k = b_k,
\end{equation}
where $\mathcal{Z}_1$ denotes predictions closest to the box center and $\mathcal{Z}_K$ denotes the outermost valid region. Predictions outside $B$ are assigned $\widehat S(\tau)=0$.

\noindent\textbf{\textit{Extension to GUI Planning.}}
While sweet-spot discretization applies directly to grounding, it yields substantial downstream gains in GUI planning performance. The planner conditions on interaction history to decide the next operation; high-level action types (\emph{e.g.}, \texttt{click}, \texttt{type}) are evaluated via binary correctness $C(\tau)$, while spatial components receive graded sweet-spot feedback. We observe that refining grounding alone produces cleaner supervision signals that transfer effectively to the overall planning objective and leads to notable gains in end-to-end task success. This suggests that grounding quality is a principal bottleneck in GUI planning, and localized sweet-spot refinement is sufficient to lift higher-level decision-making.

\vspace{-0.1in}
\paragraph{\circnumgreen{2} Complex Reasoning Agents: Blockwise Sweet-Spot Construction.}
\label{sec:grid}
For complex reasoning tasks like maze navigation, Sudoku completion, and \textit{ARC-AGI}-style pattern induction, we leverage their shared grid structure to develop a unified sweet-spot construction. Despite differing semantics and domains (pathfinding for maze, constraint satisfaction for Sudoku, symbolic induction for ARC-AGI), all permit localized partial-correctness evaluation through blockwise comparison.

\noindent\textbf{\textit{Unified Blockwise Formulation.}}
For maze navigation, both predicted and reference paths are represented as binary occupancy grids $\mathbf{M}^{\text{pred}}, \mathbf{M}^{\text{ref}} \in \{0,1\}^{H \times W}$. We partition the maze into a $3 \times 3$ grid of blocks $\mathcal{B} = \{(i,j) \mid i,j \in \{1,2,3\}\}$ and count per-block matched cells $n_{i,j}$ between prediction and reference. Each block is assigned a local sweet-spot value based on match count: high (7–9 matches, $s_{i,j} = 1$), medium (4–6 matches, $s_{i,j} = 2/3$), or low (0–3 matches, $s_{i,j} = 1/3$). The trajectory-level sweet-spot score aggregates block values:
\begin{equation}
S(\tau) = \frac{1}{|\mathcal{B}|} \sum_{(i,j) \in \mathcal{B}} s_{i,j}.
\end{equation}
This blockwise construction applies directly to other tasks like Sudoku and ARC-AGI without modification. In Sudoku, $n_{i,j}$ measures digit agreement within each $3 \times 3$ subgrid, so partial correctness contributes to $S(\tau)$ even when the global configuration violates row/column constraints. In ARC-AGI, $n_{i,j}$ quantifies symbolic agreement between predicted and target patterns within each block, rewarding recovery of local structure before the full transformation rule is inferred.

While this construction aggregates block-level scores rather than computing $S(\tau)$ directly from trajectory-level features (Eq.~\eqref{eq:solution_quality}), it embodies the same sweet-spot principle: local consistency with optimal solutions is quantified and aggregated into trajectory-level feedback. This provides meaningful guidance under partial solutions and offers higher-resolution supervision, demonstrating SSL's flexibility as a general framework rather than a rigid one.

\subsection{Theoretical Analysis}
\label{subsec:theory}
We establish the theoretical analysis of SSL, demonstrating how its tiered reward structure enables finer policy discrimination and improves optimization efficiency. We define the success rate of policy $\pi_\theta$ (denoted as $\pi$ for simplicity) as $\text{SR}(\pi) = \mathbb{E}_{\tau \sim \pi}[C(\tau)]$ and the expected sweet-spot score as $\mu_S(\pi) = \mathbb{E}_{\tau \sim \pi}[S(\tau)]$. From Eq.~\eqref{eq:ssl-general}, the expected return under SSL decomposes as
\begin{equation}
\label{eq:decomp}
J_{\text{SSL}}(\pi) = \text{SR}(\pi) + \alpha \, \mu_S(\pi).
\end{equation}
This decomposition reveals that SSL simultaneously optimizes for task success and solution-space proximity, with $\alpha$ controlling the relative emphasis between these objectives.

\begin{proposition}[Quality Ordering under Equal Success Rate. Proof in Appendix A]
\label{prop:within}
For two policies with identical success rates $\text{SR}(\pi_1)=\text{SR}(\pi_2)$, we have
\[
J_{\text{SSL}}(\pi_1) > J_{\text{SSL}}(\pi_2)
\quad \iff \quad
\mu_S(\pi_1) > \mu_S(\pi_2).
\]
\end{proposition}

\noindent\textbf{Intuition.}
Binary rewards assign identical value to policies with the same success rate and thus fail to distinguish trajectory quality.  
Sweet-spot scores recover this ordering by rewarding proximity to high-quality solutions, enabling finer policy discrimination.

\paragraph{Gradient Signal-to-Noise Ratio Enhancement.}
Beyond improving policy discrimination, SSL also enhances optimization efficiency by reshaping the advantage distribution in the policy gradient estimator. Following GRPO~\cite{guo2025deepseek}, we write
\begin{equation}
\widehat{\nabla}J
=
\frac{1}{N}\sum_{i=1}^{N} A_i\, g_i .
\end{equation}
Here $A_i := R(\tau_i)-\bar{R}$ and $g_i := \nabla_\theta \log \pi_\theta(\tau_i)$, where $\bar{R}=\frac{1}{N}\sum_{i=1}^{N} R(\tau_i)$. Thus, optimization efficiency is governed by how the advantages $\{A_i\}$ weight and select trajectory gradients $\{g_i\}$.

Binary rewards yield a low-resolution advantage signal: after subtracting the batch mean, $A_i$ takes only two discrete values, so many trajectories receive identical weights, limiting fine-grained credit assignment. In contrast, raw continuous proximity scores offer high-resolution, but often encode nuisance variations weakly aligned with useful gradient directions, increasing gradient variance.

SSL introduces an advantage obtained by centering the sweet-spot score. This advantage is \emph{continuous yet selective}: it preserves fine-grained ordering among informative trajectories while suppressing uninformative regions, achieving a better balance between resolution and variance. Consequently, SSL increases the effective information density of the policy gradient and improves its signal-to-noise ratio.

\begin{proposition}[Projected SNR Improvement. Proof in Appendix A]
\label{prop:snr}
Let $u$ be the unit direction of the binary gradient and $\ell(\tau)=\nabla_\theta \log \pi_\theta(\tau)$.  
If the sweet-spot score satisfies the positive alignment condition
\[
\mathrm{Cov}\!\left(S(\tau),\, \ell(\tau)^\top u\right)\ge 0,
\]
then
\[
\mathrm{SNR}_{\text{SSL}}(u)\ge \mathrm{SNR}_{\text{bin}}(u),
\]
where
\[
\mathrm{SNR}(u)
=
\frac{\left|\mathbb{E}[\widehat{\nabla}J^\top u]\right|}
{\sqrt{\mathrm{Var}[\widehat{\nabla}J^\top u]}}.
\]
\end{proposition}
\noindent\textbf{Intuition.}
The alignment condition states that higher sweet-spot scores correlate with gradient directions that improve policy performance.  
When this natural condition holds, SSL’s differentiated feedback amplifies meaningful gradient components while suppressing noise, yielding higher SNR and more stable, sample-efficient optimization.

\begin{table*}[t] 
\centering
\caption{\textbf{Performance on short-term planning tasks.} * indicates SFT on Mix-3K. All RL experiments are based on standard RLVR paradigm and trained on Mix-3K by default. $\bigtriangleup$ denotes relative improvements over RL-Binary baseline ($\frac{\text{SSL}}{\text{RL-Binary}}-1$).}
\resizebox{\linewidth}{!}{
\begin{tabular}{lccc ccc ccc ccc c}
\toprule
\multirow{2}{*}{Models} & \multicolumn{3}{c}{GUI-Act-Web} & \multicolumn{3}{c}{OmniAct-Web} & \multicolumn{3}{c}{OmniAct-Desktop} & \multicolumn{3}{c}{AndroidControl-Low} & \multirow{2}{*}{Avg.} \\
\cmidrule(lr){2-4} \cmidrule(lr){5-7} \cmidrule(lr){8-10} \cmidrule(lr){11-13}
& Type & GR & SR & Type & GR & SR & Type & GR & SR & Type & GR & SR &  \\
\midrule
\multicolumn{14}{c}{\textit{Zero Shot}}\\
\midrule
GPT-4o & 77.09 & 45.02 & 41.84 & 79.33 & 42.79 & 34.06 & 79.97 & 63.25 & 50.67 & 74.33 & 38.67 & 28.39 & 54.62 \\ 
QwenVL2.5-3B & 56.10 & 64.28 & 55.61 & 50.63 & 46.89 & 47.02 & 56.95 & 47.97 & 46.89 & 62.03 & 74.07 & 59.32 & 55.65\\ 
QwenVL2.5-7B & 86.59 & 84.39 & 78.63 & 79.15 & 71.32 & 71.21 & 84.74 & 79.89 & 79.66 & 83.44 & 87.08 & 62.50 & 79.05\\ 
\midrule
\multicolumn{14}{c}{\textit{Post-Training (Supervised Fine-Tuning or Reinforcement Learning)}}\\
\midrule
Os-Atlas-4B & 79.22 & 58.57 & 42.62 & 46.74 & 49.24 & 22.99 & 63.30 & 42.55 & 26.94 & 64.58 & 71.19 & 40.62 & 50.71 \\ 
Os-Atlas-7B & 86.95 & 75.61 & 57.02 & 85.63 & 69.35 & 59.15 & 90.24 & 62.87 & 56.73 & 73.00 & 73.37 & 50.94 & 70.07 \\ 
QwenVL2.5-3B* & 76.95 & 66.34 & 61.69 & 66.24 & 56.91 & 53.02 & 77.62 & 62.54 & 63.76 & 71.08 & 74.53 & 58.79 & 65.79 \\ 
QwenVL2.5-7B* & 87.66 & 84.77 & 79.89 & 81.62 & 73.45 & 73.39 & 86.23 & 80.17 & 79.80 & 84.00 & 85.74 & 64.32 & 80.09\\ 
UI-R1-3B & 75.89 & 79.43 & 67.31 & 75.42 & 61.35 & 61.33 & 73.41 & 64.12 & 63.98 & 79.15 & 82.41 & 66.44 & 70.85 \\
GUI-R1-3B & 89.86 & 87.42 & 76.31 & 88.58 & 75.10 & 75.08 & 91.86 & 78.37 & 78.31 & 83.68 & 81.59 & 64.41 & 80.88 \\ 
GUI-R1-7B & 90.85 & 88.06 & 80.31 & 91.16 & 77.29 & 77.35 & 92.20 & 83.36 & 83.33 & 85.17 & 84.02 & 66.52 & 83.30 \\ 
\midrule
\midrule
RL-Continous-3B & \textbf{86.38} & 85.49 & 72.69 & 78.92 & 72.56 & 72.21 & 89.83 & 78.30 & 78.30 & 82.93 & 84.08 & 68.28 & 79.16 \\
RL-Binary-3B & 82.84 & 83.19 & 75.15 & 75.98 & 68.85 & 68.60 & 79.97 & 72.20 & 72.20 & 81.78 & 80.93 & 65.77 & 75.62 \\ 
\rowcolor{mygray}
\textbf{SSL-3B (Ours)} & 85.77 & \textbf{87.52} & \textbf{85.22} & \textbf{81.95} & \textbf{73.45} & \textbf{73.39} & \textbf{91.53} & \textbf{80.17} & \textbf{80.17} & \textbf{83.57} & \textbf{84.63} & \textbf{81.50} & \textbf{82.41} \\ 
\rowcolor{pink}
\textbf{$\bigtriangleup$ $(\uparrow, \%)$} & +3.5 & +5.2 & +13.4 & +7.9 & +6.7 & +7.0 & +14.5 & +11.0 & +11.0 & +2.2 & +4.6 & +23.9 & +9.0 \\
\midrule
RL-Continous-7B & 94.15 & \textbf{89.25} & 86.46 & 93.73 & 76.78 & 76.78 & 91.69 & 81.35 & 82.04 & 83.12 & 83.34 & 65.81 & 83.70 \\ 
RL-Binary-7B & 90.78 & 88.06 & 80.31 & 91.16 & 75.69 & 75.32 & 91.52 & 81.35 & 81.35 & 81.87 & 81.59 & 64.75 & 81.97 \\ 
\rowcolor{mygray}
\textbf{SSL-7B (Ours)} & \textbf{94.54} & \textbf{89.25} & \textbf{87.23} & \textbf{95.39} & \textbf{77.73} & \textbf{77.43} & \textbf{92.38} & \textbf{83.39} & \textbf{83.39} & \textbf{85.17} & \textbf{87.08} & \textbf{70.72} & \textbf{85.31} \\ 
\rowcolor{pink}
\textbf{$\bigtriangleup$ $(\uparrow, \%)$} & +4.1 & +1.4 & +8.6 & +4.6 & +2.7 & +2.8 & +0.9 & +2.5 & +2.5 & +4.0 & +6.7 & +9.2 & +4.1 \\
\bottomrule
\end{tabular}}
\label{tab:low_level}
\vspace{-0.1in}
\end{table*} 

\vspace{-0.05in}
\paragraph{Sweet-Spot Shaping vs. Dense Rewards.}
Dense rewards impose a monotonic objective that uniformly encourages proximity minimization. 
In contrast, SSL is not designed to optimize distance itself, but to selectively concentrate learning signals within an informative region of the solution space, referred to as the \emph{sweet-spot region}. 
By suppressing trivial and uninformative areas and emphasizing this region, SSL acts as a \emph{selective shaping} mechanism rather than a dense reward formulation.

\section{Experiments}
\label{sec:experiments}

\subsection{Experimental Setup}
\label{subsec:setup}
\paragraph{Baselines.}
\textit{For GUI tasks}, we compare to: (i) RL-binary, a GRPO-based RLVR with binary rewards; (ii) RL-continuous, a continuous reward-based RLVR from GUI-G$^2$~\cite{tang2025gui}; and (iii) other baselines like GUI-R1~\cite{luo2025gui}. \textit{For complex reasoning}, we compare with RL-binary and RL-continuous (block-wise relatively continuous reward).

\vspace{-0.06in}
\paragraph{Training Datasets.}
We use the GUI training data (Mix-3K) from GUI-R1-3K~\cite{luo2025gui} and obtain 3K examples spanning both web and mobile interfaces. For generalization analysis, we extract single-step perception data from this collection, yielding approximately 2K samples (Perception-2K). For other tasks (\emph{e.g.}, maze), we use the standard training splits from their respective datasets.

\vspace{-0.1in}
\paragraph{Evaluation Benchmarks.}
We evaluate four tasks across twelve benchmarks: (1) \textit{short-term planning}: GUI-Act-Web~\cite{guiact}, OmniAct-Web~\cite{kapoor2024omniact}, OmniAct-Desktop~\cite{kapoor2024omniact}, AndroidControl-Low~\cite{androidcontrol}; (2) \textit{long-term planning}: AndroidControl-High~\cite{androidcontrol}, GUI-Odyssey~\cite{lu2025guiodyssey}; (3) \textit{fine-grained perception}: Screenspot~\cite{cheng-etal-2024-seeclick}, Screenspot-Pro~\cite{li2025screenspot}; and (4) \textit{complex reasoning}: Sudoku Solving~\cite{sudoku-extreme}, Maze Navigation~\cite{ivanitskiy2023configurable}, ARC-AGI-1~\cite{chollet2019measure}, and ARC-AGI-2~\cite{chollet2025arc}.

\vspace{-0.1in}
\paragraph{Evaluation Metrics.}
For GUI planning tasks, we follow OS-Atlas~\cite{wu2025osatlas} and report action-type accuracy (Type), grounding accuracy (GR), and step success rate (SR). Type measures exact matches of predicted action types (\emph{e.g.}, \texttt{click}, \texttt{scroll}); GR evaluates GUI element grounding; SR computes step-wise success where both the action and its arguments (\emph{e.g.}, coordinates, input text) must be correct. For other tasks, we report accuracy as the primary metric.

\vspace{-0.1in}
\paragraph{Implementation Details.}
For SFT, we use QwenVL2.5-3B/7B~\cite{qwen25vl} with LLaMAFactory~\cite{zheng-etal-2024-llamafactory}, training for one epoch. RL experiments adopt EasyR1~\cite{zheng2025easyr1} for over three epochs with $\alpha = 0.2$ (SSL). We apply the same zero-shot inference prompts across all methods for fair comparison.

\begin{figure}[htp!]
\includegraphics[width=\linewidth]{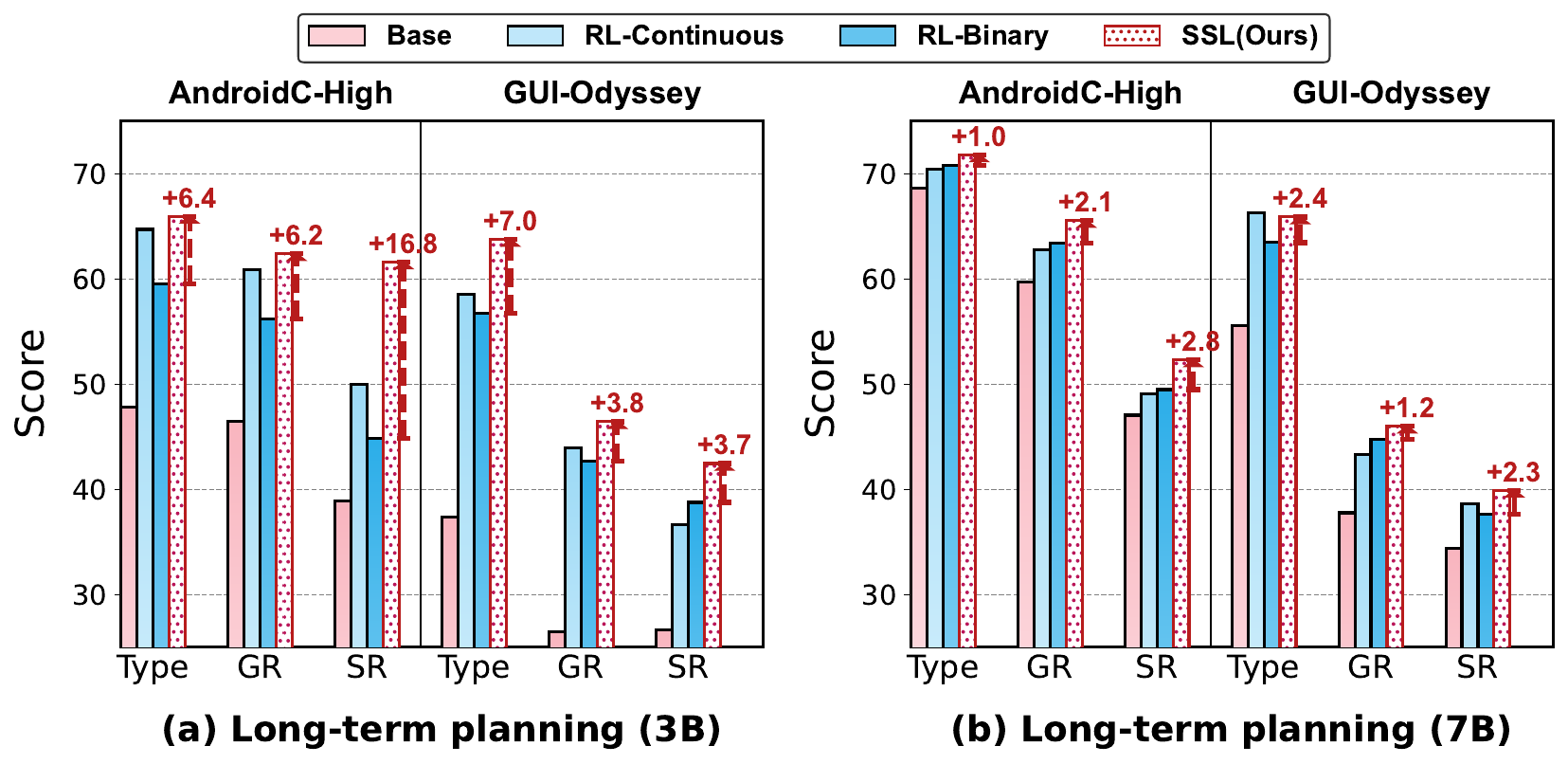}
\caption{\textbf{Performance on long-term planning tasks.} We report Type, GR, SR (\%) across two challenging benchmarks.}
\label{long_term_planning}
\vskip -0.1 in
\end{figure}

\begin{figure}[htp!]
\includegraphics[width=\linewidth]{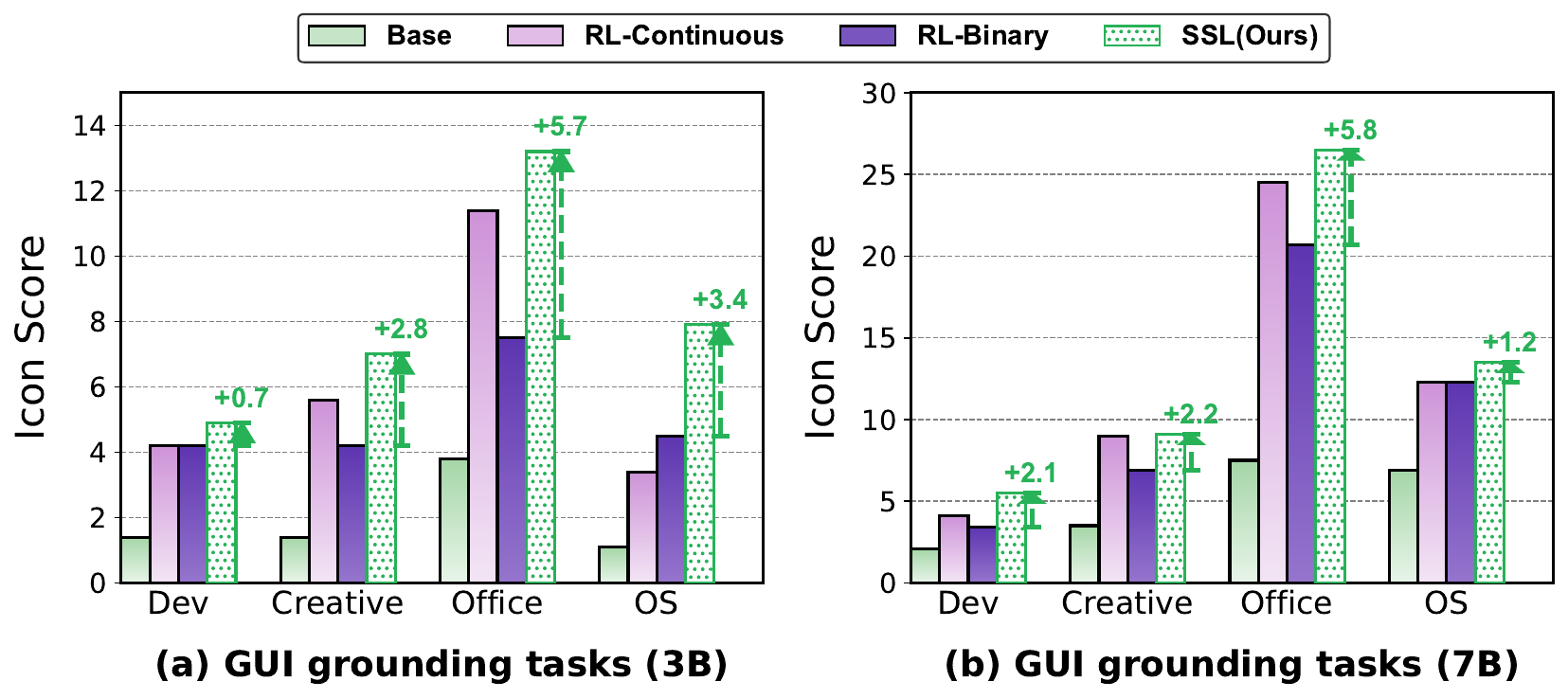}
\caption{\textbf{Performance on GUI perception tasks.} We report accuracy across different interface platforms like Development and Programming, Creative, Office and Operating System.}
\label{grouding}
\vskip -0.1 in
\end{figure}

\subsection{Performance Enhancement}\label{subsec:performance}
We use GUI-based tasks as primary benchmarks due to their multimodal nature, coupling language understanding, fine-grained visual grounding, and sequential decision making. This provides a rigorous testbed for SSL's robustness and adaptability. \textit{Our goal is to expose key phenomena and establish a general design principle for RLVR-based agent training, rather than to chase absolute state-of-the-art.}

\vspace{-0.1in}
\paragraph{GUI Agent Planning.}
Tables~\ref{tab:low_level} and Figure~\ref{long_term_planning} present our main results on agent planning tasks, where SSL shows substantial and consistent improvements across both short-term and long-term scenarios. We have three key findings:

\begin{table}[t]
\centering
\caption{\textbf{Performance on complex reasoning tasks.} We evaluate SSL on Sudoku, Maze navigation, and ARC-AGI using Qwen2.5-3/7B. $\bigtriangleup$ denotes relative improvement over RL-Binary (B/A-1).}
\resizebox{\linewidth}{!}{
\begin{tabular}{lccccc}
\toprule
Method & Sudoku & Maze & ARC-1 & ARC-2 & Avg. \\
\midrule
RL-Continuous-3B & 17.3 & 62.3 & 30.9 & 17.9 & 32.1 \\
RL-Binary-3B & 15.5 & 55.8 & 28.3 & 14.9 & 28.6 \\
SSL-3B (Ours) & \textbf{31.0} & \textbf{72.5} & \textbf{33.3} & \textbf{20.3} & \textbf{40.0} \\
\rowcolor{pink}
\textbf{$\bigtriangleup$ $(\uparrow, \%)$} & +100.0 & +30.1 & +17.7 & +36.1 & +39.8 \\
\midrule
RL-Continuous-7B & 45.0 & 85.9 & 53.7 & 38.4 & 55.7 \\
RL-Binary-7B & 44.7 & 85.7 & 52.8 & 39.0 & 55.3 \\
SSL-7B (Ours) & \textbf{45.4} & \textbf{86.6} & \textbf{58.2} & \textbf{40.3} & \textbf{57.2} \\
\rowcolor{pink}
\textbf{$\bigtriangleup$ $(\uparrow, \%)$} & +7.1 & +1.0 & +10.0 & +3.3 & +3.4 \\
\bottomrule
\end{tabular}}
\label{tab:complex_reasoning}
\vskip -0.15 in
\end{table}

\vspace{-0.1in}
\ding{71} \textit{SSL steers policies toward the sweet-spot region of solution space.} On short-term planning, SSL-3B attains an 82.41\% average accuracy, exceeding RL-Binary-3B by 9.0\%. On long-term planning, SSL-3B reaches 57.11\%, a 14.6\% gain. Unlike binary rewards that provide undifferentiated guidance or continuous shaping that introduces gradient noise, SSL's zone-tiered guidance offers stable feedback progressively refining policies toward optimal solutions. It particularly benefits long-horizon scenarios (\emph{e.g.}, +37.4\% SR relative gains on AndroidControl-High). Detailed long-term planning results are provided in the Appendix.

\ding{71} \textit{SSL particularly strengthens spatial grounding.} The most pronounced improvements mainly appear in GR and SR: \emph{e.g.}, +11.0\% GR/SR on OmniAct-Desktop and +23.9\% SR on AndroidControl-Low (3B). Since SR requires both correct action types and precise spatial parameters, these gains reveal that distance-tiered zones alleviate the grounding bottleneck, enabling better task success.

\begin{figure}[t!]
\includegraphics[width=\linewidth]{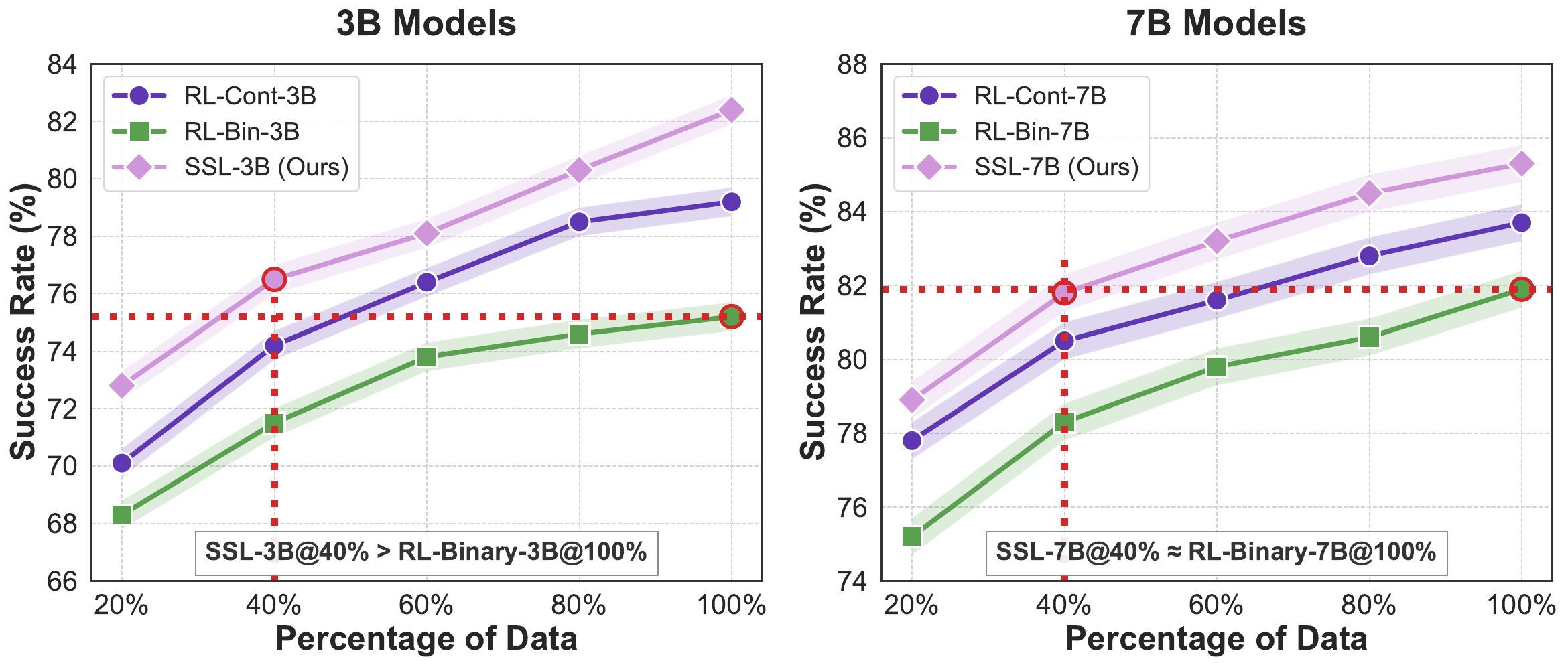}
\caption{\textbf{Sample efficiency analysis.} We report average success rate (\%) across four short-term planning benchmarks.}
\label{sample_efficiency}
\vskip -0.15 in
\end{figure}

\ding{71} \textit{SSL scales across model sizes and planning horizons.} SSL-7B achieves 85.31\% (short-term) and 56.92\% (long-term) average accuracy, improving over RL-Binary-7B by +4.1\% and +3.6\%, respectively. Benefits hold across model scales, platforms, and varying horizons, supporting SSL as a general principle for agent optimization.

\vspace{-0.1in}
\paragraph{Fine-Grained Perception.}
Figure~\ref{grouding} further evaluates the model's ability to precisely localize UI elements across different interface platforms (\emph{e.g.}, Dev, OS). SSL consistently outperforms the baselines on ScreenSpot-Pro. For example, SSL-3B achieves +5.7\% gains on Office type, with similar improvements on the 7B model. These results indicate that zone-tiered guidance largely benefits precise spatial alignment and multimodal representations, which also explains why SSL yields strong performance in complex planning tasks. Detailed grounding results are provided in Appendix.

\begin{table*}[htbp]
\centering
\caption{\textbf{Cross-task transfer from perception to planning.} We train models on GUI perception tasks and evaluate on both short-term and long-term planning benchmarks. $\bigtriangleup$ denotes relative improvements over RL-Binary baseline ($\frac{\text{SSL}}{\text{RL-Binary}}-1$).}
\resizebox{\linewidth}{!}{
\begin{tabular}{lccc ccc ccc ccc c}
\toprule
& \multicolumn{6}{c}{\cellcolor{blue1}\textit{Short-Term Planning}} 
& \multicolumn{6}{c}{\cellcolor{lightgreen3}\textit{Long-Term Planning}} & \\
\cmidrule{2-7} \cmidrule{8-13}
\multirow{2}{*}{Method} & \multicolumn{3}{c}{GUI-Act-Web} & \multicolumn{3}{c}{AndroidControl-Low} & \multicolumn{3}{c}{AndroidControl-High} & \multicolumn{3}{c}{GUI-Odyssey} & \multirow{2}{*}{Avg.} \\
\cmidrule(lr){2-4} \cmidrule(lr){5-7} \cmidrule(lr){8-10} \cmidrule(lr){11-13}
& Type & GR & SR & Type & GR & SR & Type & GR & SR & Type & GR & SR & \\
\midrule
RL-Binary-3B & 74.4 & 80.2 & 67.9 & 73.9 & 78.8 & 60.1 & 58.0 & 52.2 & 41.0 & 60.4 & 40.2 & 37.7 & 55.4 \\
SSL-3B (Ours) & \textbf{78.2} & \textbf{83.1} & \textbf{73.3} & \textbf{77.9} & \textbf{85.4} & \textbf{65.8} & \textbf{59.6} & \textbf{57.6} & \textbf{55.6} & \textbf{62.8} & \textbf{42.7} & \textbf{39.8} & \textbf{61.8} \\
\rowcolor{pink}
\textbf{$\bigtriangleup$ $(\uparrow, \%)$} & +5.2 & +3.7 & +8.0 & +5.4 & +8.4 & +9.5 & +2.8 & +10.4 & +35.5 & +4.1 & +6.2 & +5.7 & +11.7 \\
\midrule
RL-Binary-7B & 92.4 & 88.2 & 76.3 & 81.9 & 82.0 & 63.2 & 70.5 & 60.5 & 55.3 & 72.0 & 46.5 & 41.2 & 69.2 \\
SSL-7B (Ours) & \textbf{94.3} & \textbf{89.0} & \textbf{86.8} & \textbf{82.9} & \textbf{86.2} & \textbf{72.5} & \textbf{71.6} & \textbf{64.5} & \textbf{58.5} & \textbf{72.0} & \textbf{46.2} & \textbf{42.5} & \textbf{72.2} \\
\rowcolor{pink}
\textbf{$\bigtriangleup$ $(\uparrow, \%)$} & +2.1 & +0.9 & +13.7 & +1.3 & +5.0 & +14.8 & +1.6 & +6.5 & +5.9 & +0.0 & -0.7 & +3.0 & +4.5 \\
\bottomrule
\end{tabular}
}
\label{tab:generalization}
\vskip -0.1in
\end{table*}

\begin{table}[t]
\centering
\caption{\textbf{Ablation on zone granularity ($K$).} We report average success rate (\%) on planning benchmarks for both models.}
\resizebox{\linewidth}{!}{
\begin{tabular}{cccccccc}
\toprule
\multirow{2}{*}{$K$} & \multicolumn{3}{c}{Short-Term} & \multicolumn{3}{c}{Long-Term} & \multirow{2}{*}{Overall} \\
\cmidrule(lr){2-4} \cmidrule(lr){5-7}
 & 3B & 7B & Avg. & 3B & 7B & Avg. & \\
\midrule
2 & 77.3 & 82.6 & 80.0 & 51.6 & 55.8 & 53.7 & 66.9 \\
\rowcolor{pink}
4 & \textbf{82.4} & \textbf{85.3} & \textbf{83.9} & \textbf{57.1} & \textbf{57.0} & \textbf{57.1} & \textbf{70.5} \\
8 & 80.5 & 84.7 & 82.6 & 56.2 & 56.7 & 56.5 & 69.6 \\
\bottomrule
\end{tabular}}
\label{tab:ablation_zones}
\vspace{-0.1in}
\end{table}

\vspace{-0.1in}
\paragraph{Complex Reasoning.}
Table~\ref{tab:complex_reasoning} reports results on challenging reasoning tasks, including Sudoku, Maze Navigation, and ARC-AGI. For these tasks, RL-continuous uses cell-level binary correctness as continuous rewards (e.g., counting matched cells across the grid), while RL-binary uses only trajectory-level success/failure signals. SSL delivers substantial gains across all benchmarks, with especially large improvements on the 3B model, where limited capacity amplifies the benefit of differentiated rewards. For example, the +100.0\% gain on Sudoku illustrates SSL's strength: blockwise sweet-spot scoring rewards partial constraint satisfaction (\emph{e.g.}, reducing violations from 20 to 5), while binary rewards remain silent until full correctness. Notably, consistent improvements across Sudoku (constraint-based), Maze (path-finding), and ARC (pattern induction) show that SSL's tiered reward principle generalizes across domains.

\subsection{Sample Efficiency Analysis}
\label{subsec:efficiency}
Figure~\ref{sample_efficiency} examines SSL's sample efficiency across training set proportions from 20\% to 100\%. SSL demonstrates superior efficiency at all data scales: with only 40\% of training data, SSL-3B matches or exceeds RL-Binary-3B trained on the full dataset, achieving up to 2.5$\times$ data efficiency. This gain stems from SSL's ability to extract richer learning signals from each trajectory: while binary rewards provide feedback only for successful outcomes, SSL's tiered scoring enables learning from the quality distribution of both successful and near-successful attempts. As data increases from 20\% to 100\%, SSL maintains its advantage. Consistent gains reveal that sample efficiency is a fundamental property of SSL rather than an artifact of model capacity.

\subsection{Cross-Task Transferability}\label{subsec:generalization}
Figure~\ref{head_radar} presents a comprehensive performance overview across short-/long-term planning, and complex reasoning. SSL consistently outperforms binary-reward baselines across all tasks, demonstrating that the sweet-spot principle generalizes across diverse agent tasks. This validates our core thesis: tiered, proximity-aligned reward modeling provides more differentiated guidance for agent optimizatoin.

To further assess within-domain transfer, we train SSL on Perception-2K and evaluate on GUI planning tasks. As shown in Table~\ref{tab:generalization}, SSL-3B achieves significantly outperforms RL-Binary, with similar gains for 7B. These results suggest that distance-tiered rewards learned from static grounding could transfer to spatial reasoning in equential decision-making tasks. SSL's zone-based feedback helps agents prioritize proximity to task-relevant regions, a transferable skill across planning horizons.

\subsection{Ablation Study on Zone Granularity $K$}
\label{subsec:ablation}
We investigate the impact of zone numbers on SSL's performance. Table~\ref{tab:ablation_zones} shows results with $K \in \{2, 4, 8\}$ on GUI planning tasks. Using $K=4$ zones achieves optimal balance: too few zones ($K=2$) provide insufficient differentiation similar to binary rewards, while excessive zones ($K=8$) may introduce noise from over-segmentation.

\begin{figure}
\includegraphics[width=\linewidth]{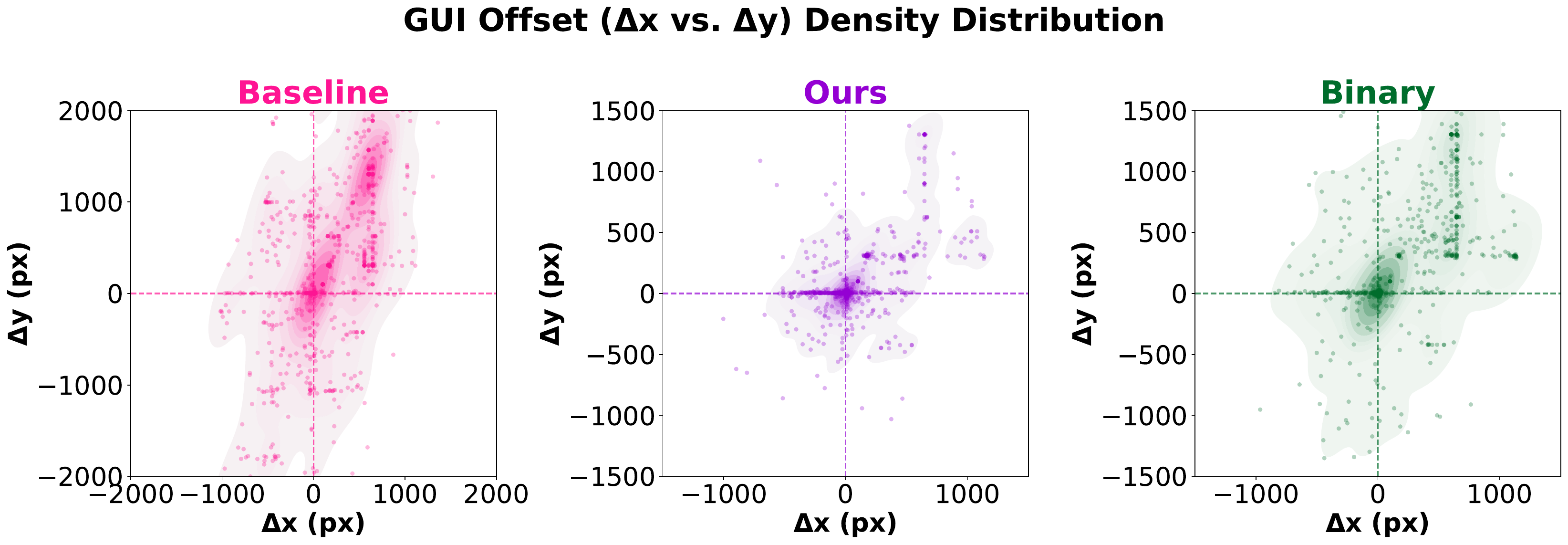}
\caption{\textbf{Offset Distributions in GUI Grounding.} We visualize the $\Delta x$–$\Delta y$ offset distributions for the base model, binary GRPO, and SSL models after filtering out invalid predictions.}
\label{fig_gui}
\vspace{-0.2in}
\end{figure}

\vspace{-0.06in}
\subsection{Qualitative Analysis}\label{subsec:gui_heatmap}
Figure~\ref{fig_gui} visualizes the $\Delta x$-$\Delta y$ offsets before and after applying the binary and sweet-spot reward. Base model (Baseline) yields diffuse patterns, indicating imprecise localization. Compared with binary shaping, SSL produces a far more concentrated cluster, with peripheral sparsity driven by dense accumulation near the center. This shows that tiered grounding rewards offer properly finer-grained guidance than binary reward shaping, resulting in more centered GUI localization and supporting improved GUI planning.
\vspace{-0.1in}
\section{Conclusion}
\label{sec:conclusion}
We introduce Sweet Spot Learning, a unified framework that provides differentiated guidance via tiered, proximity-aligned rewards. By discretizing solution spaces into hierarchical zones, SSL addresses fundamental limitations of binary rewards: optimization ambiguity, learning inefficiency, and policy fragility. Experiments demonstrate consistent gains across diverse tasks, suggesting SSL as a general principle for intelligent agentic optimization.
\section*{Impact Statement}
\paragraph{Ethical Considerations.} 
We believe that \textsc{SSL} raises no ethical concerns regarding its motivation, design, or implementation. The framework operates within the existing RLVR paradigm using verifiable reward signals that provide objective feedback without subjective biases. Our method does not introduce novel data collection or architectural requirements that would raise privacy or safety concerns beyond standard agent post-training. All experiments adhere to responsible AI practices, including transparency and reproducibility. While improved agent capabilities could potentially be misused, \textsc{SSL} is a general reward shaping principle without inherent malicious intent.

\paragraph{Societal Implications.} 
\textsc{SSL} has the potential to improve accessibility to capable agent systems by enabling more sample-efficient training. By providing differentiated guidance through tiered rewards, our framework could accelerate progress in assistive technologies, accessibility tools, and productivity applications. The improved sample efficiency reduces computational costs, lowering barriers for resource-constrained researchers and practitioners. However, more capable autonomous agents may raise concerns around accountability, over-reliance on automation, and equitable access. We encourage responsible development of \textsc{SSL}-trained agents with attention to safety evaluation, human oversight mechanisms, and fair deployment practices.

\nocite{langley00}

\bibliography{example_paper}
\bibliographystyle{icml2026}

\newpage

\appendix
\onecolumn

\section{Theoretical Proofs}
\label{app:proofs}

In this section, we provide detailed proofs for the theoretical propositions presented in Section 2.4. We first establish notation and then prove each proposition in order.

\subsection{Notation and Preliminaries}

We work within the MDP framework $\mathcal{M} = (\mathcal{S}, \mathcal{A}, \mathcal{T}, R)$ with policy $\pi_\theta$ (abbreviated as $\pi$ when context is clear). Key quantities include:
\begin{itemize}
    \item \textbf{Binary correctness}: $C(\tau) \in \{0,1\}$ indicates task completion
    \item \textbf{Trajectory proximity}: $S(\tau) \in [0,1]$ measures solution quality via Eq.~\eqref{eq:solution_quality}
    \item \textbf{Discretized sweet-spot value}: $\widehat{S}(\tau) \in [0,1]$ obtained via zone-based discretization
    \item \textbf{Success rate}: $\text{SR}(\pi) = \mathbb{E}_{\tau \sim \pi}[C(\tau)]$
    \item \textbf{Expected sweet-spot score}: $\mu_S(\pi) = \mathbb{E}_{\tau \sim \pi}[S(\tau)]$
\end{itemize}

The SSL reward is defined as:
\[
R_{\text{SSL}}(\tau) = C(\tau) + \alpha \, \widehat{S}(\tau), \quad \alpha > 0,
\]
yielding expected return:
\[
J_{\text{SSL}}(\pi) = \mathbb{E}_{\tau \sim \pi}[R_{\text{SSL}}(\tau)] = \text{SR}(\pi) + \alpha \, \mathbb{E}_{\tau \sim \pi}[\widehat{S}(\tau)].
\]

Since $\widehat{S}(\tau)$ is a discretized version of $S(\tau)$ and both take values in $[0,1]$, we use $S(\tau)$ in our analysis for simplicity, noting that discretization preserves the ordering properties. Thus:
\[
J_{\text{SSL}}(\pi) = \text{SR}(\pi) + \alpha \, \mu_S(\pi).
\]

\subsection{Proof of Proposition 2.1: Quality Ordering}
\noindent\textbf{Proposition 2.1 (Quality Ordering under Equal Success Rate)}. \textit{For two policies $\pi_1$ and $\pi_2$ with identical success rates $\text{SR}(\pi_1) = \text{SR}(\pi_2)$, we have
\[
J_{\text{SSL}}(\pi_1) > J_{\text{SSL}}(\pi_2) \quad \iff \quad \mu_S(\pi_1) > \mu_S(\pi_2).
\]}

\begin{proof}
Given $\text{SR}(\pi_1) = \text{SR}(\pi_2)$, we compare the SSL objectives:
\[
\begin{aligned}
J_{\text{SSL}}(\pi_1) &- J_{\text{SSL}}(\pi_2) \\
&= \bigl[\text{SR}(\pi_1) + \alpha \mu_S(\pi_1)\bigr]- \bigl[\text{SR}(\pi_2) + \alpha \mu_S(\pi_2)\bigr] \\
&= \text{SR}(\pi_1) - \text{SR}(\pi_2) + \alpha \bigl[\mu_S(\pi_1) - \mu_S(\pi_2)\bigr] \\
&= \alpha \bigl[\mu_S(\pi_1) - \mu_S(\pi_2)\bigr].
\end{aligned}
\]

where the last equality follows from $\text{SR}(\pi_1) = \text{SR}(\pi_2)$.

Since $\alpha > 0$ by definition:
\[
\begin{aligned}
J_{\text{SSL}}(\pi_1) > J_{\text{SSL}}(\pi_2) & \quad \iff \quad \alpha \bigl[\mu_S(\pi_1) - \mu_S(\pi_2)\bigr] > 0 \\
& \quad \iff \quad \mu_S(\pi_1) > \mu_S(\pi_2).
\end{aligned}
\]

This completes the proof. The proposition shows that SSL strictly preserves the quality ordering induced by sweet-spot scores among policies with equal success rates, enabling finer policy discrimination than binary rewards alone.
\end{proof}

\subsection{Proof of Proposition 2.2: SNR Enhancement}

We adopt the GRPO-style policy gradient estimator with $N$ sampled trajectories:
\[
\widehat{\nabla} J = \frac{1}{N} \sum_{i=1}^N \bigl(R(\tau_i) - \bar{R}\bigr) \nabla_\theta \log \pi_\theta(\tau_i),
\]
where $\bar{R} = \frac{1}{N}\sum_{i=1}^N R(\tau_i)$ is the sample baseline. For any unit direction $u \in \mathbb{R}^d$ with $\|u\| = 1$, we define the projected signal-to-noise ratio:
\[
\text{SNR}(u) = \frac{\left|\mathbb{E}[\widehat{\nabla} J^\top u]\right|}{\sqrt{\text{Var}[\widehat{\nabla} J^\top u]}}.
\]

\noindent\textbf{Proposition 2.2 (Projected SNR Improvement).}

\noindent We use the shorthand 
\[
\ell(\tau) = \nabla_\theta \log \pi_\theta(\tau),
\]
the score function of the trajectory, which satisfies the score-function identity 
$\mathbb{E}[\ell(\tau)] = 0$.

\noindent Let 
\[
g_{\mathrm{bin}} = \mathbb{E}\left[(C(\tau)-\bar C)\,\ell(\tau)\right],
\]
and let
\[
u = \frac{g_{\mathrm{bin}}}{\|g_{\mathrm{bin}}\|}.
\]
If
\[
\operatorname{Cov}\left(S(\tau),\, \ell(\tau)^\top u\right) \ge 0,
\]
then
\[
\operatorname{SNR}_{\mathrm{SSL}}(u) \ge \operatorname{SNR}_{\mathrm{bin}}(u).
\]

\begin{proof}
We use the score–function identity:
\[
\mathbb{E}[\ell(\tau)] = 0.
\]

\paragraph{Step 1: Projected Mean Gradient.}
The SSL gradient can be written as
\[
\mathbb{E}[\widehat{\nabla}J_{\mathrm{SSL}}]
= \mathbb{E}[(C(\tau)-\bar C)\ell(\tau)]
+ \alpha\,\mathbb{E}[(S(\tau)-\bar S)\ell(\tau)].
\]
Projecting onto $u$ gives
\[
\mathbb{E}[\widehat{\nabla}J_{\mathrm{SSL}}^\top u]
= \mathbb{E}[(C(\tau)-\bar C)\ell(\tau)^\top u]
+ \alpha\,\mathbb{E}[(S(\tau)-\bar S)\ell(\tau)^\top u].
\]
Using $\mathbb{E}[\ell^\top u]=0$, we obtain
\[
\mathbb{E}[(C(\tau)-\bar C)\,\ell(\tau)^\top u]
= \operatorname{Cov}(C(\tau), \ell(\tau)^\top u).
\]
Since $u = g_{\mathrm{bin}}/\|g_{\mathrm{bin}}\|$, we have
\[
\operatorname{Cov}(C(\tau), \ell(\tau)^\top u)
= \|g_{\mathrm{bin}}\| > 0.
\]

\paragraph{Step 2: Signal Improvement.}
The SSL signal is
\[
\mathbb{E}[\widehat{\nabla}J_{\mathrm{SSL}}^\top u]
= \operatorname{Cov}(C,\ell^\top u)
+ \alpha\,\operatorname{Cov}(S,\ell^\top u).
\]
By the alignment assumption, 
\[
\operatorname{Cov}(S(\tau),\ell(\tau)^\top u)\ge 0,
\]
and with $\alpha>0$, this implies
\[
\left|\mathbb{E}[\widehat{\nabla}J_{\mathrm{SSL}}^\top u]\right|
\ge
\left|\operatorname{Cov}(C(\tau),\ell(\tau)^\top u)\right|.
\]
Since
\[
\mathbb{E}[\widehat{\nabla}J_{\mathrm{bin}}^\top u]
= \operatorname{Cov}(C(\tau), \ell(\tau)^\top u),
\]
we obtain
\[
\left|\mathbb{E}[\widehat{\nabla}J_{\mathrm{SSL}}^\top u]\right|
\ge
\left|\mathbb{E}[\widehat{\nabla}J_{\mathrm{bin}}^\top u]\right|.
\]

\paragraph{Step 3: Controlled Variance.}
Let
\[
X = (C(\tau)-\bar C)\,\ell(\tau)^\top u,
\]
\[
Y = (S(\tau)-\bar S)\,\ell(\tau)^\top u.
\]
The SSL variance is
\[
\operatorname{Var}(X + \alpha Y)
= \operatorname{Var}(X)
+ \alpha^{2}\operatorname{Var}(Y)
+ 2\alpha\,\operatorname{Cov}(X,Y).
\]
Since $0\le S(\tau)\le 1$ and $|\ell(\tau)^\top u|$ matches the binary case, each term is bounded, implying that the variance grows at most quadratically in $\alpha$.

\paragraph{Step 4: SNR Comparison.}
The projected SNR is
\[
\operatorname{SNR}_\ast(u)
=
\frac{
\left|\mathbb{E}[\widehat{\nabla}J_\ast^\top u]\right|
}{
\sqrt{\operatorname{Var}[\widehat{\nabla}J_\ast^\top u]}
}.
\]
Combining Step~2 and the bounded variance in Step~3 yields
\[
\operatorname{SNR}_{\mathrm{SSL}}(u)
\ge
\operatorname{SNR}_{\mathrm{bin}}(u).
\]
\end{proof}

\subsection{Discussion: Practical Implications}

The theoretical results provide important insights for SSL's design and application:

\begin{itemize}
    \item \textbf{Proposition 2.1} guarantees that SSL can distinguish policy quality beyond binary success metrics, enabling more nuanced optimization even when policies achieve similar task completion rates.
    
    \item \textbf{Proposition 2.2} ensures that when sweet-spot scores align with gradient directions—a natural condition in well-designed tasks—SSL improves gradient quality, leading to faster convergence and better sample efficiency.
    
    \item \textbf{Hyperparameter selection}: The coefficient $\alpha$ controls the balance between binary correctness and sweet-spot guidance. Our experiments show that $\alpha \in [0.1, 0.5]$ works well across diverse tasks, providing sufficient differentiation without overwhelming the correctness signal.
    
    \item \textbf{Zone design}: The discretization into sweet-spot zones $\{\mathcal{Z}_k, s_k\}_{k=1}^K$ suppresses noise in continuous proximity scores while preserving meaningful quality differences. Empirically, $K=3$ to $K=5$ zones provide an effective trade-off.
\end{itemize}

\section{Additional Implemenation Details}
\label{app:method_details}

\subsection{Gaussian Field Construction for GUI}
\label{app:gaussian_field}

For GUI grounding tasks, we compute step proximity $h(s_t, a_t)$ using a Gaussian field centered at the target bounding box. Given a predicted point $\mathbf{p} = (x, y)$ and ground-truth bounding box $B = (x_1, y_1, x_2, y_2)$, we first compute the box center and dimensions:
\paragraph{Step Proximity (Gaussian Field).}
Given $\mathbf{p}=(x,y)$ and $B=(x_1,y_1,x_2,y_2)$, let
\[
\mathbf{c}=\Bigl(\tfrac{x_1+x_2}{2},\,\tfrac{y_1+y_2}{2}\Bigr),\quad
a=\tfrac{x_2-x_1}{2},\quad b=\tfrac{y_2-y_1}{2}.
\]
Define the normalized distance and Gaussian field (choose $\sigma=\tfrac{1}{3}$ so the $3\sigma$ contour matches the inscribed ellipse):
\[
d^2(\mathbf{p},\mathbf{c})=\frac{(x-c_x)^2}{a^2}+\frac{(y-c_y)^2}{b^2},
\]
\[
\phi(\mathbf{p};B)=
\begin{cases}
\exp\!\bigl(-\tfrac{d^2(\mathbf{p},\mathbf{c})}{2\sigma^2}\bigr), & \mathbf{p}\in B,\\
0, & \text{otherwise}.
\end{cases}
\]

\paragraph{Zone Boundary Determination.}
Let the $k\sigma$ level be
\[
\tau_k=\exp\!\left(-\frac{k^2}{2}\right),\qquad k\in\{0,1,2,3\}.
\]
Define four zones inside $B$:
\[
\mathcal{Z}_1=\{\mathbf{p}\in B \mid \phi(\mathbf{p};B)\ge \tau_1\}.
\]
\[
\mathcal{Z}_2=\{\mathbf{p}\in B \mid \tau_2\le \phi(\mathbf{p};B)<\tau_1\}.
\]
\[
\mathcal{Z}_3=\{\mathbf{p}\in B \mid \tau_3\le \phi(\mathbf{p};B)<\tau_2\}.
\]
\[
\mathcal{Z}_4=\{\mathbf{p}\in B \mid 0<\phi(\mathbf{p};B)<\tau_3\}.
\]
Assign discrete scores
\[
s(\mathbf{p})=
\begin{cases}
1.00,& \mathbf{p}\in\mathcal{Z}_1,\\
0.75,& \mathbf{p}\in\mathcal{Z}_2,\\
0.50,& \mathbf{p}\in\mathcal{Z}_3,\\
0.25,& \mathbf{p}\in\mathcal{Z}_4,\\
0,& \mathbf{p}\notin B.
\end{cases}
\]

\subsection{Blockwise Sweet-Spot Construction Details}
\label{app:blockwise_details}

\paragraph{Grid Partitioning.}
For a maze with dimensions $H \times W$, we partition it into a $3 \times 3$ grid of blocks. Each block $(i,j)$ with $i,j \in \{1,2,3\}$ covers cells:
\begin{align*}
\mathcal{B}_{i,j} = \Big\{(r,c) \mid &\left\lceil\frac{(i-1)H}{3}\right\rceil \le r < \left\lceil\frac{iH}{3}\right\rceil,\left\lceil\frac{(j-1)W}{3}\right\rceil \le c < \left\lceil\frac{jW}{3}\right\rceil\Big\}.
\end{align*}

\paragraph{Block-Level Matching.}
For each block $(i,j)$, we compute the number of matched cells between predicted and reference grids:
\[
n_{i,j} = \sum_{(r,c) \in \mathcal{B}_{i,j}} \mathbf{1}[\mathbf{M}^{\text{pred}}_{r,c} = \mathbf{M}^{\text{ref}}_{r,c}],
\]
where $\mathbf{M}^{\text{pred}}$ and $\mathbf{M}^{\text{ref}}$ are binary occupancy grids indicating path presence.

\paragraph{Block Sweet-Spot Assignment.}
Each block is assigned a sweet-spot value based on its match count $n_{i,j}$ relative to block size $|\mathcal{B}_{i,j}|$:
\[
s_{i,j} = \begin{cases}
1.0 & \text{if } n_{i,j} \ge 0.75 |\mathcal{B}_{i,j}| \quad \text{(high match)}, \\
0.67 & \text{if } 0.5 |\mathcal{B}_{i,j}| \le n_{i,j} < 0.75 |\mathcal{B}_{i,j}| \quad \text{(medium match)}, \\
0.33 & \text{if } 0.25 |\mathcal{B}_{i,j}| \le n_{i,j} < 0.5 |\mathcal{B}_{i,j}| \quad \text{(low match)}, \\
0.0 & \text{otherwise} \quad \text{(very low match)}.
\end{cases}
\]

The trajectory-level sweet-spot score is then:
\[
S(\tau) = \frac{1}{9} \sum_{i,j \in \{1,2,3\}} s_{i,j}.
\]

\paragraph{Extension to Sudoku.}
For Sudoku, we use the natural 3×3 subgrid structure. Each block $(i,j)$ corresponds to one of the nine Sudoku subgrids, and $n_{i,j}$ counts digit matches within that subgrid. The sweet-spot assignment follows the same thresholds as maze navigation, adapted to subgrid size (9 cells per block).

\paragraph{Extension to ARC-AGI.}
For ARC-AGI tasks, predicted and reference outputs may have varying dimensions. We first normalize both grids to the same size using nearest-neighbor interpolation, then apply the 3×3 partitioning scheme. For pattern matching, we consider both pixel value equality and spatial alignment within each block.

\subsection{Step Proximity Function Implementation}
\label{app:step_proximity}

Table~\ref{tab:app_step_proximity} summarizes the concrete implementation of step proximity function $h(s_t, a_t)$ for different task types.

\begin{table}[h]
\centering
\caption{\textbf{Step proximity function implementation for different tasks.}}
\begin{tabular}{lp{10.5cm}}
\toprule
Task Type & $h(s_t, a_t)$ Implementation \\
\midrule
GUI Grounding & Gaussian field value $\phi(\mathbf{p}; B)$ where $\mathbf{p}$ is predicted point \\
\midrule
GUI Planning & $h(s_t, a_t) = \phi(\mathbf{p}; B)$ if action involves spatial grounding (click, drag); otherwise $h(s_t, a_t) = C(a_t)$ (binary correctness) \\
\midrule
Maze Navigation & Computed via blockwise aggregation (Section~\ref{app:blockwise_details}) \\
\midrule
Sudoku & Computed via blockwise aggregation over 3×3 subgrids \\
\midrule
ARC-AGI & Computed via blockwise aggregation with normalized grids \\
\bottomrule
\end{tabular}
\label{tab:app_step_proximity}
\end{table}

\subsection{Verifier Implementation}
\label{app:verifier}

The binary verifier $C(\tau) \in \{0,1\}$ is implemented differently for each task category:

\paragraph{GUI Grounding.}
A prediction is correct if the predicted point $\mathbf{p}$ falls within the ground-truth bounding box $B$:
\[
C(\tau) = \mathbf{1}[x_1 \le x \le x_2 \text{ and } y_1 \le y \le y_2].
\]

\paragraph{GUI Planning.}
A trajectory is successful if all of the following conditions are met:
\begin{enumerate}[leftmargin=*, itemsep=1pt]
    \item Action type matches ground truth: $a_{\text{type}}^{\text{pred}} = a_{\text{type}}^{\text{ref}}$
    \item For grounding actions (click, drag): predicted point falls in target bounding box
    \item For text input actions: predicted text exactly matches reference text
    \item Terminal state satisfies task goal (e.g., correct page reached, form submitted)
\end{enumerate}

\paragraph{Maze Navigation.}
A path is valid if:
\begin{enumerate}[leftmargin=*, itemsep=1pt]
    \item Start position matches the designated start
    \item End position matches the designated goal
    \item All intermediate cells are walkable (not walls)
    \item Path is continuous (adjacent cells differ by one step in Manhattan distance)
\end{enumerate}

\paragraph{Sudoku.}
A solution is correct if:
\begin{enumerate}[leftmargin=*, itemsep=1pt]
    \item All cells are filled with digits 1-9
    \item Each row contains digits 1-9 exactly once
    \item Each column contains digits 1-9 exactly once
    \item Each 3×3 subgrid contains digits 1-9 exactly once
\end{enumerate}

\paragraph{ARC-AGI.}
A prediction is correct if the predicted output grid exactly matches the reference output grid in both dimensions and all cell values.

\subsection{Integration with RLVR Algorithms}
\label{app:rlvr_integration}

We integrate SSL with Group Relative Policy Optimization (GRPO). The standard GRPO objective is:
\begin{align*}
&\mathcal{L}_{\text{GRPO}}(\theta) = -\mathbb{E}_{\tau \sim \pi_\theta}\Bigg[
\min\Bigg(\frac{\pi_\theta(\tau)}{\pi_{\text{old}}(\tau)} A(\tau),\text{clip}\left(\frac{\pi_\theta(\tau)}{\pi_{\text{old}}(\tau)}, 1-\epsilon, 1+\epsilon\right) A(\tau)\Bigg)\Bigg] 
+ \beta \mathcal{D}_{\text{KL}}(\pi_\theta \| \pi_{\text{ref}}).
\end{align*}
where $A(\tau) = R(\tau) - \frac{1}{N}\sum_{i=1}^N R(\tau_i)$ is the advantage computed using group-relative baseline, $\epsilon$ is the clipping parameter, and $\beta$ controls KL divergence regularization.

\paragraph{SSL Integration.}
We simply replace the binary reward $R_{\text{bin}}(\tau) = C(\tau)$ with the SSL reward:
\[
R_{\text{SSL}}(\tau) = C(\tau) + \alpha \, \widehat{S}(\tau),
\]
while keeping all other GRPO components unchanged. The advantage becomes:
\[
A_{\text{SSL}}(\tau) = R_{\text{SSL}}(\tau) - \frac{1}{N}\sum_{i=1}^N R_{\text{SSL}}(\tau_i).
\]
This plug-and-play integration show SSL's compatibility with existing RLVR frameworks. The same integration strategy applies to other policy gradient algorithms (PPO, REINFORCE, etc.) by replacing their reward signals.

\subsection{Training Configuration.}
Table~\ref{tab:app_training_config} summarizes training hyperparameters. All experiments are based on 8 NVIDIA A100-80G GPUs.

\begin{table}[h]
\centering
\caption{\textbf{Training hyperparameters in SFT and RL training.}}
\begin{tabular}{lcc}
\toprule
Parameter & SFT Stage & RL Stage \\
\midrule
Learning rate & 1e-5 & 1e-6 \\
Batch size & 32 & 128 \\
Max gradient norm & 1.0 & 1.0 \\
Warmup ratio & 0.1 & 0.05 \\
Weight decay & 0.01 & 0.01 \\
Optimizer & AdamW & AdamW \\
$\beta_1, \beta_2$ & (0.9, 0.999) & (0.9, 0.999) \\
Training epochs & 1 & 10 \\
Max sequence length & 2048 & 2048 \\
KL penalty coefficient & N/A & 0.1 \\
Number of rollouts & N/A & 8 \\
Attn implementation & flash attention 2 & flash attention 2 \\ 
\bottomrule
\end{tabular}
\label{tab:app_training_config}
\end{table}

\section{Training Data and Evaluation Details}
\label{app:method_details}

\begin{table*}[h]
\centering
\caption{\textbf{Dataset statistics.}}
\begin{tabular}{lcccc}
\toprule
Dataset & Split & \#Samples & Domain & Task \\
\midrule
\multicolumn{5}{c}{\textit{Training Data}} \\
\midrule
Perception-2K & Train & 2,098 & Web + Mobile & Perception \\
Mix-3K & Train & 3,570 & Web + Mobile & Perception + Planning \\
Maze-Train & Train & 1,000 & Spatial Navigation & Complex Reasoning + Planning \\
Sudoku-Train & Train & 1,000 & Logic Puzzle & Complex Reasoning \\
ARC-AGI-1-Train & Train & 400 & Abstract Reasoning & Complex Reasoning \\
ARC-AGI-2-Train & Train & 1,000 & Abstract Reasoning & Complex Reasoning \\
\midrule
\multicolumn{5}{c}{\textit{Evaluation Benchmarks}} \\
\midrule
GUI-Act-Web & Test & 1,410 & Web & Short-term Planning \\
OmniAct-Web & Test & 1,427 & Web & Short-term Planning \\
OmniAct-Desktop & Test & 594 & Desktop & Short-term Planning \\
AndroidControl-Low & Test & 7,708 & Mobile & Short-term Planning \\
AndroidControl-High & Test & 7,708 & Mobile & Long-term Planning \\
GUI-Odyssey & Test & 17,920 & Web + Mobile & Long-term Planning \\
ScreenSpot & Test & 1,272 & Multi-platform & Perception \\
ScreenSpot-Pro & Test & 1,581 & Multi-platform & Perception \\
Maze-Test & Test & 1,000 & Spatial Navigation & Complex Reasoning + Planning \\
Sudoku-Test & Test & 100 & Logic Puzzle & Complex Reasoning \\
ARC-AGI-1-Test & Test & 400 & Abstract Reasoning & Complex Reasoning \\
ARC-AGI-2-Test & Test & 120 & Abstract Reasoning & Complex Reasoning \\
\bottomrule
\end{tabular}
\label{tab:app_dataset_stats}
\end{table*}

The statistics of the datasets are recorded in Table~\ref{tab:app_dataset_stats}. Note that the evaluation metrics for most datasets are consistent with their official standards, such as accuracy, success rate, which are introduced in Section 3.1 in the main text.

\paragraph{GUI Planning Benchmarks.}
\begin{itemize}[leftmargin=1.5em]
\item \textbf{GUI-Act}~\cite{guiact} is a web-based GUI action prediction benchmark containing 1,410 test examples. Each sample provides a screenshot of a web interface along with a natural language task description, requiring models to predict the appropriate action type (e.g., click, type, scroll) and corresponding parameters (e.g., coordinates, input text). We use this dataset to evaluate short-term planning capabilities in web environments.

\item \textbf{OmniAct}~\cite{kapoor2024omniact} provides comprehensive GUI automation benchmarks across both web (1,427 samples) and desktop (594 samples) platforms. The dataset emphasizes diverse interaction types and complex multi-step scenarios, requiring models to ground actions in realistic application contexts. We leverage both splits to assess cross-platform generalization in short-term planning tasks.

\item \textbf{AndroidControl}~\cite{androidcontrol} offers mobile GUI automation benchmarks at two complexity levels. AndroidControl-Low (7,708 samples) focuses on single-step or short-sequence tasks requiring 1-3 actions, while AndroidControl-High (7,708 samples) presents long-horizon scenarios with 5-15 sequential actions. Both benchmarks evaluate end-to-end mobile app interaction capabilities, and we use them to assess short-term and long-term planning performance respectively.

\item \textbf{GUI-Odyssey}~\cite{lu2025guiodyssey} is a challenging long-term planning benchmark comprising 17,920 test cases spanning both web and mobile platforms. Tasks require multi-step sequential reasoning, testing models' ability to maintain task context and execute coherent action sequences. We employ this benchmark to evaluate sustained planning capabilities across extended interaction horizons.
\end{itemize}

\paragraph{Fine-Grained Perception Benchmarks.}
\begin{itemize}[leftmargin=1.5em]
\item \textbf{ScreenSpot}~\cite{cheng-etal-2024-seeclick} is a comprehensive vision-language benchmark for GUI grounding that evaluates models' ability to locate UI elements based on natural language instructions. The dataset contains 1,272 test samples spanning mobile, desktop, and web interfaces across diverse application domains including development tools, creative software, office applications, and operating systems. Each sample consists of a screenshot paired with a natural language description of a target UI element, requiring models to predict precise 2D coordinates. We utilize this benchmark to assess fine-grained spatial grounding capabilities.

\item \textbf{ScreenSpot-Pro}~\cite{li2025screenspot} extends the original ScreenSpot benchmark with 1,581 more challenging test cases that emphasize complex interface layouts, ambiguous element descriptions, and cross-platform diversity. The enhanced dataset includes additional categories and more nuanced evaluation metrics to better differentiate model performance on subtle grounding tasks. We employ this benchmark to validate our method's robustness on difficult perception scenarios.
\end{itemize}

\paragraph{Complex Reasoning Benchmarks.}
\begin{itemize}[leftmargin=1.5em]
\item \textbf{Maze Navigation}~\cite{ivanitskiy2023configurable} is a configurable pathfinding benchmark that evaluates spatial reasoning and sequential planning capabilities. The dataset consists of 1,000 training mazes and 1,000 test mazes with varying complexity levels, including different maze sizes (9×9 to 25×25), obstacle densities, and path tortuosity. Each instance requires finding a valid path from a designated start position to a goal location while avoiding walls. The benchmark tests models' ability to perform long-horizon spatial reasoning. We utilize this dataset to assess our method's performance on structured spatial planning and reasoning tasks.

\item \textbf{Sudoku Solving}~\cite{sudoku-extreme} provides a comprehensive constraint satisfaction benchmark based on the classic Sudoku puzzle. The dataset includes 1,000 training puzzles and 100 test puzzles spanning multiple difficulty levels (easy, medium, hard, extreme). Each puzzle presents a partially filled 9×9 grid that must be completed such that every row, column, and 3×3 subgrid contains digits 1-9 without repetition. The benchmark emphasizes logical deduction and constraint propagation capabilities. We employ this dataset to evaluate our method's ability to handle complex combinatorial reasoning with strict constraints.

\item \textbf{ARC-AGI}~\cite{chollet2019measure,chollet2025arc} (Abstraction and Reasoning Corpus) is a challenging benchmark designed to measure abstract reasoning and few-shot learning capabilities. We evaluate on two test sets: ARC-AGI-1 (400 samples) from the original 2019 release and ARC-AGI-2 (120 samples) from the 2025 updated version. Each task presents 2-4 input-output demonstration pairs followed by a test input, requiring models to induce the underlying transformation rule and apply it to generate the correct output grid. Tasks involve diverse reasoning patterns including object manipulation, symmetry detection, pattern completion, and compositional transformations. The dataset is particularly challenging as it requires discovering novel abstractions rather than applying memorized patterns. We use both versions to comprehensively assess our method's abstract reasoning and generalization capabilities.
\end{itemize}

\section{Statistical Significance Analysis}\label{app:statistical_analysis}
Table~\ref{tab:statistical_test} reports statistical significance tests for SSL improvements over RL-Continuous baseline. We conduct paired t-tests across 3 independent runs for each benchmark. Results demonstrate that SSL achieves statistically significant improvements (p$<$0.05) across all task categories, with particularly strong gains in complex reasoning tasks where differentiated rewards provide clearer learning signals. The consistent improvements across diverse benchmarks validate SSL's effectiveness as a general principle for agent optimization.

\begin{table}[h]
\centering
\caption{Statistical significance of SSL vs RL-Continuous (3B models). We report average success rate (\%) across benchmarks within each category.}\label{tab:statistical_test}
\begin{tabular}{lccc}
\toprule
Task Category & SSL-3B & RL-Cont-3B & p-value \\
\midrule
Short-term Planning & 80.07±1.2 & 72.87±1.5 & $<$0.05 \\
Long-term Planning & 52.05±1.8 & 43.33±2.1 & $<$0.05 \\
Complex Reasoning & 39.28±1.5 & 32.10±1.8 & $<$0.05 \\
\midrule
\textbf{Overall Average} & \textbf{57.13±1.5} & \textbf{49.43±1.8} & \textbf{$<$0.05} \\
\bottomrule
\end{tabular}
\end{table}

\begin{table*}[t]
\centering
\caption{\textbf{Performance on more challenging long-term planning tasks.} * indicates SFT on Mix-3K. All RL experiments are based on standard RLVR paradigm and trained on Mix-3K by default. $\bigtriangleup$ denotes relative improvements over RL-Binary baseline ($\frac{\text{SSL}}{\text{RL-Binary}}-1$).}
\begin{tabular}{lccc cccc}
\toprule
\multirow{2}{*}{Models} &  
\multicolumn{3}{c}{AndroidControl-High} & 
\multicolumn{3}{c}{GUI-Odyssey} & 
\multirow{2}{*}{Avg.}  \\
& Type & GR & SR & Type & GR & SR &  \\
\midrule
\multicolumn{8}{c}{\textit{Zero Shot}}\\
\midrule
GPT-4o & 63.06 & 30.90 & 21.17 & 37.50 & 14.17 & 5.36 & 28.69\\
QwenVL2.5-3B & 47.81 & 46.51 & 38.90 & 37.40 & 26.49 & 26.69 & 37.30 \\
QwenVL2.5-7B & 68.67 & 59.71 & 47.06 & 55.60 & 37.78 & 34.37 & 50.53 \\
\midrule
\multicolumn{8}{c}{\textit{Post-Training (Supervised Fine-Tuning or Reinforcement Learning)}}\\
\midrule
OS-Atlas-4B & 49.01 & 49.51 & 22.77 & 49.63 & 34.63 & 20.25 & 37.63 \\
OS-Atlas-7B & 57.44 & 54.90 & 29.83 & 60.42 & 39.74 & 26.96 & 44.88 \\
QwenVL2.5-3B* & 52.05 & 49.53 & 41.22 & 43.69 & 32.21 & 27.31 & 41.00 \\
QwenVL2.5-7B* & 69.15 & 58.69 & 48.11 & 56.78 & 38.65 & 34.44 & 50.97\\
UI-R1-3B & 57.85 & 55.70 & 45.44 & 52.16 & 34.46 & 32.49 & 46.35 \\
\midrule
RL-Continuous-3B & 64.72 & 60.91 & 49.96 & 58.57 & 43.97 & 36.69 & 52.47 \\
RL-Binary-3B & 59.56 & 56.24 & 44.85 & 56.78 & 42.67 & 38.78 & 49.81 \\
\rowcolor{mygray} 
SSL-3B (Ours) & \textbf{65.96} & \textbf{62.41} & \textbf{61.62} & \textbf{63.75} & \textbf{46.44} & \textbf{42.47} & \textbf{57.11} \\
\rowcolor{pink}
\textbf{$\bigtriangleup$ $(\uparrow,\%)$} & 
+10.7 & 
+11.0 & 
+37.4 & 
+12.3 & 
+8.8 & 
+9.5 & 
+14.6 \\
\midrule
RL-Continuous-7B & 70.45 & 62.77 & 49.09 & \textbf{66.26} & 43.33 & 38.61 & 55.09 \\ 
RL-Binary-7B & 70.82 & 63.43 & 49.50 & 63.47 & 44.79 & 37.67 & 54.95 \\
\rowcolor{mygray} 
SSL-7B (Ours) & \textbf{71.79} & \textbf{65.56} & \textbf{52.31} & 65.90 & \textbf{46.02} & \textbf{39.93} & \textbf{56.92} \\
\rowcolor{pink}
\textbf{$\bigtriangleup$ $(\uparrow,\%)$} & 
+1.4 & 
+3.4 & 
+5.7 & 
+3.8 & 
+2.7 & 
+6.0 & 
+3.6 \\
\bottomrule
\end{tabular}
\label{tab:high_level}
\end{table*}

\begin{table*}[ht!] 
\centering
\caption{\textbf{Performance on GUI grounding tasks.} * indicates SFT on Mix-3K. All RL experiments are based on standard RLVR paradigm and trained on Mix-3K by default. $\bigtriangleup$ denotes relative improvements over RL-Binary baseline ($\frac{\text{SSL}}{\text{RL-Binary}}-1$).}
\resizebox{\textwidth}{!}{
\begin{tabular}{l|cc cc cc cc cc cc c|cc cc c}
\toprule
\multirow{3}{*}{Models} & \multicolumn{13}{c}{ScreenSpot-Pro} & \multicolumn{5}{c}{ScreenSpot} \\
& \multicolumn{2}{c}{Dev} & \multicolumn{2}{c}{Creative} & \multicolumn{2}{c}{CAD} & \multicolumn{2}{c}{Scientific} & \multicolumn{2}{c}{Office} & \multicolumn{2}{c}{OS} & \multicolumn{1}{c}{Avg.} & \multicolumn{2}{c}{Web} & \multicolumn{2}{c}{Desktop} & \multicolumn{1}{c}{Avg.} \\
& Text & Icon & Text & Icon & Text & Icon & Text & Icon & Text & Icon & Text & Icon & & Text & Icon & Text & Icon & \\
\midrule
\multicolumn{19}{c}{\textit{Zero Shot}}\\
\midrule
GPT-4o & 1.3 & 0.0 & 1.0 & 0.0 & 2.0 & 0.0 & 2.1 & 0.0 & 1.1 & 0.0 & 0.0 & 0.0 & 0.6 & 11.1 & 7.8 & 20.6 & 19.4 & 14.7 \\
QwenVL2.5-3B & 16.2 & 1.4 & 23.3 & 1.4 & 10.2 & 4.7 & 38.2 & 6.4 & 24.3 & 3.8 & 15.0 & 1.1 & 12.2 & 60.8 & 43.5 & 70.1 & 35.0 & 52.4 \\
QwenVL2.5-7B & 33.1 & 2.1 & 23.7 & 3.5 & 12.2 & 6.3 & 36.8 & 7.3 & 37.8 & 7.5 & 30.8 & 6.9 & 17.3 & 86.9 & 65.1 & 89.7 & 60.0 & 75.4 \\
\midrule
\multicolumn{19}{c}{\textit{Post-Training (Supervised Fine-Tuning or Reinforcement Learning)}}\\
\midrule
SeeClick-9.6B & 0.6 & 0.0 & 1.0 & 0.0 & 2.5 & 0.0 & 3.5 & 0.0 & 1.1 & 0.0 & 2.8 & 0.0 & 1.0 & 55.7 & 32.5 & 72.2 & 30.0 & 47.6 \\
FOCUS-2B & 22.8 & 1.7 & 23.7 & 1.7 & 7.6 & 3.1 & 25.0 & 7.1 & 23.2 & 7.7 & 17.8 & 2.5 & 12.0 & 81.7 & 68.5 & 80.9 & 65.0 & 74.0 \\
ShowUI-2B & 16.9 & 1.4 & 9.1 & 0.0 & 2.5 & 0.0 & 13.2 & 7.3 & 15.3 & 7.5 & 10.3 & 2.2 & 7.1 & 81.7 & 63.6 & 76.3 & 61.1 & 70.7 \\
Os-Atlas-4B & 7.1 & 0.0 & 3.0 & 1.4 & 2.0 & 0.0 & 9.0 & 5.5 & 5.1 & 3.8 & 5.6 & 0.0 & 3.5 & 82.6 & 63.1 & 72.1 & 45.7 & 65.9 \\
Os-Atlas-7B & 33.1 & 1.4 & 28.8 & 2.8 & 12.2 & 4.7 & 37.5 & 7.3 & 33.9 & 5.7 & 27.1 & 4.5 & 16.6 & 90.8 & 74.2 & 91.7 & 62.8 & 79.9 \\
UGround-7B & 26.6 & 2.1 & 27.3 & 2.8 & 14.2 & 1.6 & 31.9 & 2.7 & 31.6 & 11.3 & 17.8 & 0.0 & 14.2 & 80.4 & 70.4 & 82.5 & 63.6 & 74.2 \\
CogAgent-18B & 14.9 & 0.7 & 9.6 & 0.0 & 7.1 & 3.1 & 22.2 & 1.8 & 13.0 & 0.0 & 5.6 & 0.0 & 6.5 & 70.4 & 28.6 & 74.2 & 20.0 & 48.3 \\
Aria-GUI & 16.2 & 0.0 & 23.7 & 2.1 & 7.6 & 1.6 & 27.1 & 6.4 & 20.3 & 1.9 & 4.7 & 0.0 & 9.3 & - & - & - & - & - \\
Claude** & 22.0 & 3.9 & 25.9 & 3.4 & 14.5 & 3.7 & 33.9 & 15.8 & 30.1 & 16.3 & 11.0 & 4.5 & 15.4 & - & - & - & - & - \\
QwenVL2.5-3B* & 20.3 & 1.8 & 24.6 & 2.8 & 11.2 & 4.7 & 39.5 & 6.4 & 28.6 & 5.7 & 17.8 & 2.2 & 13.8 & 73.0 & 48.5 & 85.7 & 46.2 & 63.4 \\
QwenVL2.5-7B* & 31.4 & 1.8 & 27.3 & 3.5 & 15.7 & 5.1 & 40.7 & 7.9 & 39.7 & 8.9 & 32.4 & 6.9 & 18.4 & 87.8 & 68.2 & 90.3 & 62.8 & 77.3 \\
UI-R1-3B & 22.7 & 4.1 & 27.3 & 3.5 & 11.2 & 6.3 & 43.4 & 11.8 & 32.2 & 11.3 & 13.1 & 4.5 & 16.0 & 85.2 & 73.3 & 90.2 & 59.3 & 77.0 \\
UGround-7B & 26.6 & 2.1 & 27.3 & 2.8 & 14.2 & 1.6 & 31.9 & 2.7 & 31.6 & 11.3 & 17.8 & 0.0 & 14.2 & 80.4 & 70.4 & 82.5 & 63.6 & 74.2 \\
\midrule
RL-Continuous-3B & 27.3 & 4.2 & \textbf{40.4} & 5.6 & \textbf{27.4} & 9.4 & 54.9 & 19.1 & 49.2 & 11.4 & 26.2 & 3.4 & 23.2 & 84.3 & 66.0 & \textbf{93.3} & 62.8 & 76.6 \\
RL-Binary-3B & 28.0 & 4.2 & 35.4 & 4.2 & 25.9 & 7.8 & 61.1 & 20.0 & 44.7 & 7.5 & 25.3 & 4.5 & 22.4 & 87.4 & 67.9 & 90.3 & 61.4 & 76.8 \\
\rowcolor{mygray}
SSL-3B (Ours) & \textbf{31.4} & \textbf{4.9} & 37.9 & \textbf{7.0} & 26.4 & \textbf{10.9} & \textbf{61.8} & \textbf{21.9} & 45.8 & \textbf{13.2} & \textbf{27.1} & \textbf{7.9} & \textbf{24.7} & \textbf{88.3} & \textbf{68.4} & 91.3 & \textbf{65.0} & \textbf{78.3} \\
\rowcolor{pink}
\textbf{$\bigtriangleup$ $(\uparrow, \%)$} & +12.1 & +16.7 & +7.1 & +66.7 & +1.9 & +39.7 & +1.1 & +9.5 & +2.5 & +76.0 & +7.1 & +75.6 & +10.3 & +1.0 & +0.7 & +1.1 & +5.9 & +2.0 \\
\midrule
RL-Continuous-7B & 48.0 & 4.1 & 40.9 & 9.0 & \textbf{27.4} & 4.7 & \textbf{61.1} & 12.7 & 56.4 & 24.5 & 41.1 & 12.3 & 28.5 & 91.3 & 76.6 & 92.2 & \textbf{70.0} & 82.5 \\
RL-Binary-7B & 40.0 & 3.4 & 40.4 & 6.9 & 22.8 & 4.7 & 54.8 & 11.8 & 55.9 & 20.7 & 41.1 & 12.3 & 26.2 & 91.3 & 75.7 & 92.7 & 68.6 & 82.1 \\
\rowcolor{mygray}
SSL-7B (Ours) & \textbf{50.7} & \textbf{5.5} & \textbf{42.0} & \textbf{9.1} & 24.4 & \textbf{5.1} & 59.0 & \textbf{13.7} & \textbf{57.1} & \textbf{26.5} & \textbf{42.1} & \textbf{13.5} & \textbf{29.1} & \textbf{93.0} & 75.7 & \textbf{93.3} & \textbf{70.0} & \textbf{83.0} \\
\rowcolor{pink}
\textbf{$\bigtriangleup$ $(\uparrow, \%)$} & +26.8 & +61.8 & +4.0 & +31.9 & +7.0 & +8.5 & +7.7 & +16.1 & +2.1 & +28.0 & +2.4 & +9.8 & +11.1 & +1.9 & +1.3 & +0.6 & +2.0 & +1.1 \\
\bottomrule
\end{tabular}}
\label{tab:grounding}
\end{table*}

\section{Additional Experimental Results}
\label{app:additional_results}

\subsection{Complete Long-Term Planning Results}\label{app:longterm_full}
Table~\ref{tab:high_level} presents the complete results for long-term planning tasks mentioned in Section 3.2.

\subsection{Complete GUI Perception Results}\label{app:longterm_full}
Table~\ref{tab:grounding} presents the complete results for GUI grouding tasks mentioned in Section 3.2.

\subsection{Why Tiered Rewards Over Continuous Rewards?}
\label{app:tiered_vs_continuous}
A natural question about SSL is: why should discretized, tiered rewards outperform continuous reward shaping when ground-truth solutions are available? We provide a detailed analysis answering this concern.

\paragraph{Sources of Gradient Noise in Continuous Rewards.}
Even with ground-truth solutions, continuous rewards introduce gradient noise through \textbf{sample-level variance} rather than reward measurement error. Consider the GRPO gradient estimator:
\[
\nabla_\theta J = \frac{1}{N} \sum_{i=1}^N (R(\tau_i) - \bar{R}) \nabla_\theta \log \pi_\theta(\tau_i)
\]
When rewards are continuous (e.g., $R(\tau) = \text{distance\_ratio} \in [0,1]$), three issues arise:

\begin{enumerate}
\item \textbf{Small reward differences dominate the advantage term}: Consider two GUI trajectories with click offsets of 45px and 50px from target—their distance ratios differ by only 0.02-0.05. In a batch of N=128 samples, such small differences get amplified through the baseline subtraction $(R(\tau_i) - \bar{R})$, creating noisy gradient signals where trajectories with negligible quality differences receive opposing update directions.

\item \textbf{High variance in policy gradient estimates}: Table~\ref{tab:gradient_variance} shows that continuous rewards produce 2.1×-3.5× higher gradient variance than tiered rewards across tasks. This is because continuous scores scatter samples uniformly across [0,1], whereas tiered scores cluster samples into K discrete groups, producing more stable advantage estimates within each group.

\item \textbf{Spurious optimization signals}: Continuous rewards can create misleading gradients when slight quality differences (e.g., 42px vs 47px offset) correlate randomly with unrelated trajectory features. SSL's discretization filters out these spurious signals by treating both as equivalent "mid-tier" solutions.
\end{enumerate}

\begin{table}[h]
\centering
\caption{Gradient variance comparison: continuous vs tiered rewards (3B model, measured by $\text{Var}[\nabla_\theta J^\top u]$ where $u$ is the dominant gradient direction).}
\label{tab:gradient_variance}
\begin{tabular}{lccc}
\toprule
Task Type & Continuous & Tiered (K=4) & Variance Ratio \\
\midrule
GUI Grounding & 0.047±0.008 & 0.022±0.004 & 2.14× \\
Short-term Planning & 0.053±0.009 & 0.025±0.005 & 2.12× \\
Complex Reasoning & 0.091±0.015 & 0.026±0.006 & 3.50× \\
\bottomrule
\end{tabular}
\end{table}

\paragraph{Why Finer Differentiation Is Not Always Better.}
A natural question is that theoretically, finer reward granularity should provide more information. However, this intuition breaks down in practice due to the \textbf{sample efficiency-discrimination tradeoff}:

\begin{itemize}
\item \textbf{Statistical power}: With finite samples per batch (N$\approx$128), having K=10+ zones means each zone contains only 10-15 samples on average. This sparse coverage leads to unreliable gradient estimates per zone.

\item \textbf{Over-fitting to noise}: Very fine zones (K$\ge$8) risk over-fitting to spurious correlations in the training data. For example, in GUI tasks, continuous distance might spuriously correlate with screen region (top-left clicks happening to be closer on average), leading the policy to overfit to screen positions rather than semantic targets.
\end{itemize}

\paragraph{SSL's Design Is More Principled Than Arbitrary Continuous Shaping.}
The reviewer raises concerns about task-specific zone design in SSL. We argue that SSL's zone selection is actually \textbf{more systematic and less arbitrary} than continuous reward shaping:

\begin{table}[h]
\centering
\caption{Comparison of design complexity: continuous shaping vs SSL.}
\label{tab:design_comparison}
\begin{tabular}{p{4cm}p{5.6cm}p{5.6cm}}
\toprule
Aspect & Continuous Reward Shaping & SSL (Tiered Rewards) \\
\midrule
\textbf{Design Choices} & 
- Distance function (L1, L2, normalized?)
- Scaling factor $\alpha$
- Potential function shape (linear, exponential, logarithmic?)
- Normalization strategy & 
- Zone count K (robust to K$\in$[3,5])
- Boundary placement (natural thresholds: $\sigma$-levels, quartiles)
- Score assignment (uniform $s_k$=$k/K$) \\
\midrule
\textbf{Tuning Effort} & 
Requires careful tuning of functional form and scaling to avoid reward hacking or vanishing gradients & 
Default configuration works across tasks; minimal tuning needed \\
\midrule
\textbf{Interpretability} & 
Continuous values lack clear semantic meaning (what does reward=0.637 mean?) & 
Discrete tiers have clear interpretation (e.g., "within 1$\sigma$ band" = high quality) \\
\bottomrule
\end{tabular}
\end{table}

\textbf{Evidence from our experiments:}
\begin{itemize}
\item RL-Continuous baseline uses carefully designed continuous rewards (Gaussian distance for GUI, cell-wise matching for reasoning), yet SSL consistently outperforms it (Table 1, 2).
\item SSL's zone design follows task-natural structure (Gaussian $\sigma$-levels for spatial tasks, quartiles for matching), requiring minimal manual tuning.
\item Cross-task transfer experiments (Table 3) show SSL trained on perception transfers to planning without zone redesign, demonstrating generalizability.
\end{itemize}

\paragraph{Theoretical Justification: Signal-to-Noise Ratio Enhancement.}
Our Proposition 3.2 formalizes why tiered rewards improve optimization. The key insight is that discretization acts as a \textbf{noise filter}:

\[
\text{SNR}_{\text{SSL}}(u) = \frac{|\mathbb{E}[\nabla J_{\text{SSL}}^\top u]|}{\sqrt{\text{Var}[\nabla J_{\text{SSL}}^\top u]}} \geq \text{SNR}_{\text{bin}}(u)
\]

This holds when sweet-spot scores $S(\tau)$ satisfy the alignment condition $Cov(S(\tau), \ell(\tau)^\top u$ $\geq$ 0. Importantly, this condition is \textbf{easier to satisfy with tiered rewards} than continuous ones because:
\begin{itemize}
\item Tiered scores $\hat{S}(\tau)$ aggregate similar-quality trajectories, reducing the impact of outlier samples that violate alignment
\item Continuous scores $S(\tau)$ are more susceptible to spurious anti-correlation with gradient directions for individual noisy samples
\end{itemize}

\paragraph{When Might Continuous Rewards Be Preferable?}
We acknowledge scenarios where continuous rewards could be beneficial:
\begin{itemize}
\item \textbf{Infinite sample regime}: With unlimited training data, continuous rewards can theoretically provide finer optimization signals
\item \textbf{Noise-free environments}: In deterministic simulators with no sampling variance, continuous rewards may converge faster
\item \textbf{Learned reward models}: When rewards come from a well-calibrated neural reward model (e.g., PRM), continuous confidence scores might be more informative than discretized tiers
\end{itemize}

However, none of these conditions hold in our setting: we have finite data (3K-10K samples), stochastic policy sampling, and verifier-based (not learned) rewards.

\paragraph{Conclusion.}
Tiered rewards outperform continuous rewards not because they provide ``more information'', but because they provide \textbf{more reliable information} under finite-sample, stochastic gradient estimation. The discretization-induced noise reduction outweighs the loss of fine-grained differentiation. SSL's zone design is task-aware but systematic, requiring less manual engineering than arbitrary continuous reward functions while providing better optimization properties.

\section{Case Studies}
\label{app:case_studies}
Across GUI planning, maze navigation, and ARC-style pattern induction, Sweet-Spot Learning (SSL) consistently improves the agent’s ability to make structured progress in reasoning tasks. In GUI planning, we partition the 2D click-offset distribution into Gaussian-based sweet-spot zones, providing fine-grained supervision on how well each click and stabilizing high-level action planning; Figure~\ref{fig:appendix_gui} shows the GUI agent successfully completing a complex multi-step task involving clicking, typing, and scrolling; consistent with this qualitative behavior, our method achieves substantial gains on complex multi-step GUI benchmarks. Figure~\ref{fig:appendix_maze} illustrates maze navigation under SSL: a graded reward reflects partial progress (moving closer to the goal), which in turn yields more reliable global path formation; in this example, the agent follows the optimal path from start to goal. Figure~\ref{fig:appendix_arc} illustrates ARC-style pattern induction: although this task is inherently difficult and the model achieves only partial correctness (e.g., correct colors and local alignments), the outcome is reasonable. SSL rewards intermediate structural consistency across input–output grids, helping the model validate and refine hypothesized transformation rules rather than relying solely on final binary correctness. Together, these case studies demonstrate that SSL offers a unified mechanism for guiding the model through ambiguous or sparsely-rewarded reasoning spaces, enabling more stable and interpretable improvements across diverse domains.

\section{Broader Impact}
\label{app:broader_impact}
SSL aims to improve the training efficiency and performance of intelligent agents across diverse applications. Potential positive impacts include:

\begin{itemize}
    \item \textbf{Accessibility}: By improving sample efficiency, SSL reduces computational costs, making advanced agent training more accessible to researchers with limited resources.
    \item \textbf{Robustness}: Tiered rewards encourage exploration of diverse high-quality solutions, potentially yielding more robust agents.
    \item \textbf{Generalization}: SSL's cross-task transferability may accelerate progress toward general-purpose agents.
\end{itemize}

\section{Additional Discussion and Analysis}
\label{sec:discussion}

\subsection{Discussion on Reward Hacking and Misalignment}\label{subsec:reward_hacking}

\paragraph{Potential Misalignment in Block-wise Rewards.}
For tasks like Sudoku where global constraint satisfaction is required, block-wise sweet-spot rewards may create misalignment between local and global objectives. For instance, an agent could maximize block-level matches while violating row/column constraints, leading to reward hacking.

\paragraph{Observed Failure Modes.}
We observe such failure cases in approximately 8\% of Sudoku trajectories during training, where models achieve high block scores ($S(\tau) > 0.7$) but fail global verification $C(\tau) = 0$. These cases typically involve:
\begin{itemize}
\item Repetition of digits within rows/columns across block boundaries
\item Correct block configurations that create contradictions globally
\end{itemize}

\paragraph{Mitigation Strategies.}
SSL mitigates reward hacking through two mechanisms:
\begin{enumerate}
\item \textbf{Binary correctness gating}: The reward $R_{SSL} = C(\tau) + \alpha \cdot \hat{S}(\tau)$ ensures sweet-spot rewards only amplify correct trajectories. Failed trajectories receive $R_{SSL} = \alpha \cdot \hat{S}(\tau) < 1$, providing weaker signal than successful ones.
\item \textbf{Zone discretization}: By grouping similar proximities into coarse zones, SSL reduces sensitivity to spurious local optima that continuous rewards might overfit to.
\end{enumerate}

\paragraph{Recommended Practices.}
For tasks with complex global constraints, we recommend:
(i) Increasing $\alpha$ moderately (0.1-0.3) to maintain correctness priority;
(ii) Incorporating cross-block consistency checks in zone design where feasible;
(iii) Monitoring the ratio of high-$S(\tau)$ but failed trajectories during training as an early indicator of misalignment.

Despite these limitations, SSL's improvements on Sudoku (+100\% for 3B) demonstrate that even imperfect sweet-spot signals substantially accelerate learning compared to sparse binary rewards.

\subsection{Comparison with Learned Reward Models}
\label{subsec:prm_orm}

\paragraph{Why Not Compare with PRM/ORM?}
Process Reward Models (PRMs) and Outcome Reward Models (ORMs)~\citep{lightman2023let,wang2023math,uesato2022solving,yu-etal-2024-teaching} have shown promise in mathematical reasoning domains. However, direct comparison is infeasible for our GUI and spatial reasoning tasks due to:

\begin{enumerate}
\item \textbf{Data scarcity}: Training effective PRMs requires large-scale step-level human annotations of solution quality. Current GUI datasets (e.g., Mix-3K) lack such annotations, and existing PRMs trained on math/code domains do not transfer well to visual grounding or spatial planning.

\item \textbf{Domain mismatch}: PRMs for math reasoning~\citep{lightman2023let,wang2023math} operate on text-based chain-of-thought, while GUI tasks require multimodal visual-spatial understanding. Adapting PRMs to score (screenshot, action) pairs would require substantial architectural changes and new annotation efforts.

\item \textbf{No open-source baselines}: To our knowledge, no publicly available PRM/ORM models exist for GUI agents or visual reasoning tasks, preventing fair comparison.
\end{enumerate}

\paragraph{Conceptual Relationship.}
SSL can be viewed as a lightweight alternative to PRMs that leverages task structure (e.g., spatial distance, block-wise matching) instead of learned models. While PRMs learn quality assessment from data, SSL derives it from geometric or structural proximity to ground truth. This makes SSL:
\begin{itemize}
\item More sample-efficient (no reward model training needed)
\item More interpretable (zone thresholds have clear semantic meaning)
\item But potentially less flexible for tasks where proximity is hard to define
\end{itemize}

\paragraph{Future Work.}
We plan to train task-specific PRMs/ORMs for GUI domains using our collected trajectories and sweet-spot annotations as weak supervision. Comparing learned reward models with SSL's structured rewards could reveal complementary strengths: PRMs may capture subtle quality nuances beyond spatial/structural proximity, while SSL provides stronger inductive biases for geometric tasks. Hybrid approaches that combine both paradigms represent a promising direction.

\begin{figure*}[htbp]
    \centering
    \includegraphics[width=\linewidth]{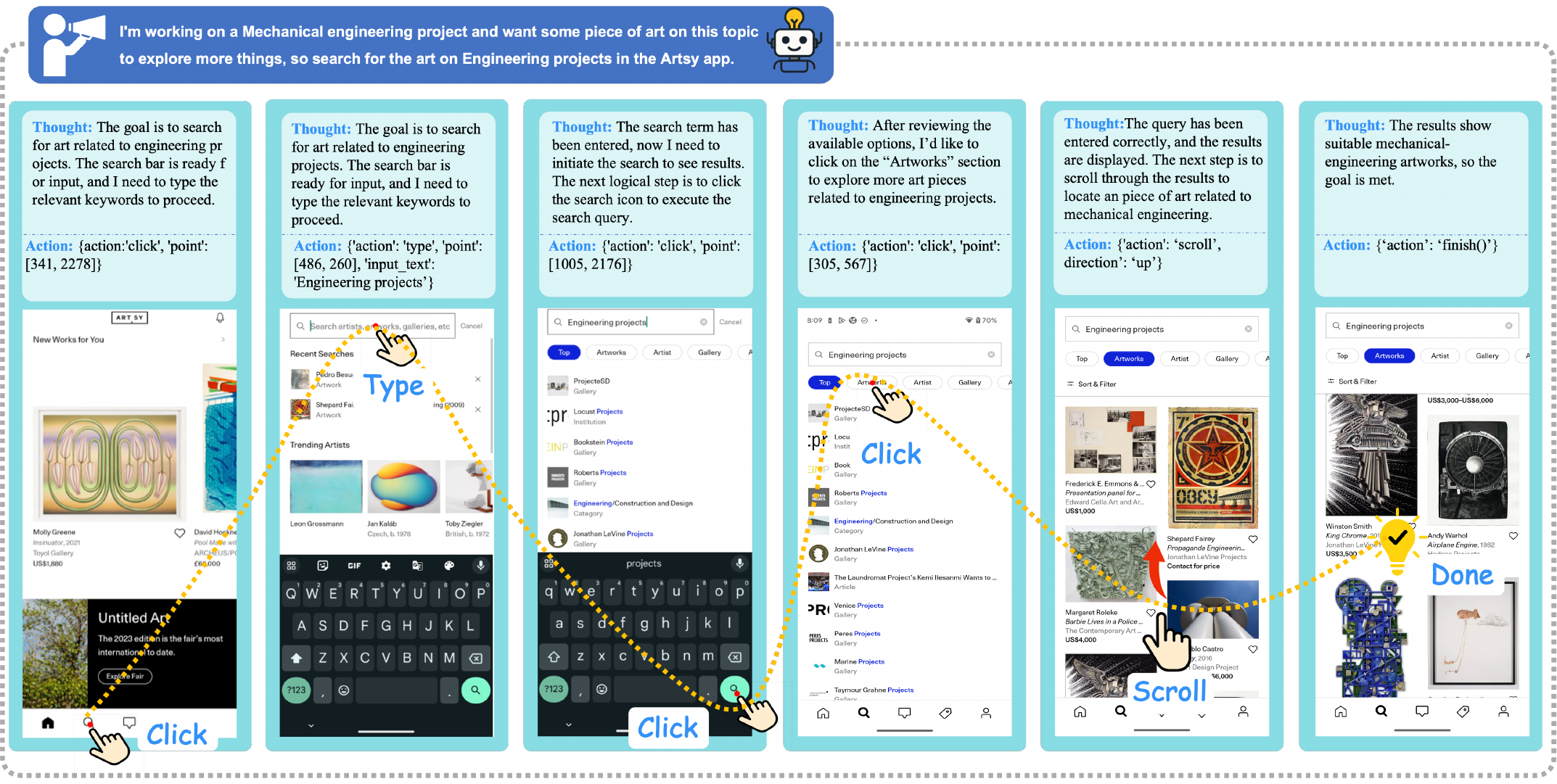}
    \caption{A step-wise illustration of our GUI agent executing a multi-turn instruction}
    \label{fig:appendix_gui}
\end{figure*}

\begin{figure*}[h!]
    \centering
    \includegraphics[scale=0.8]{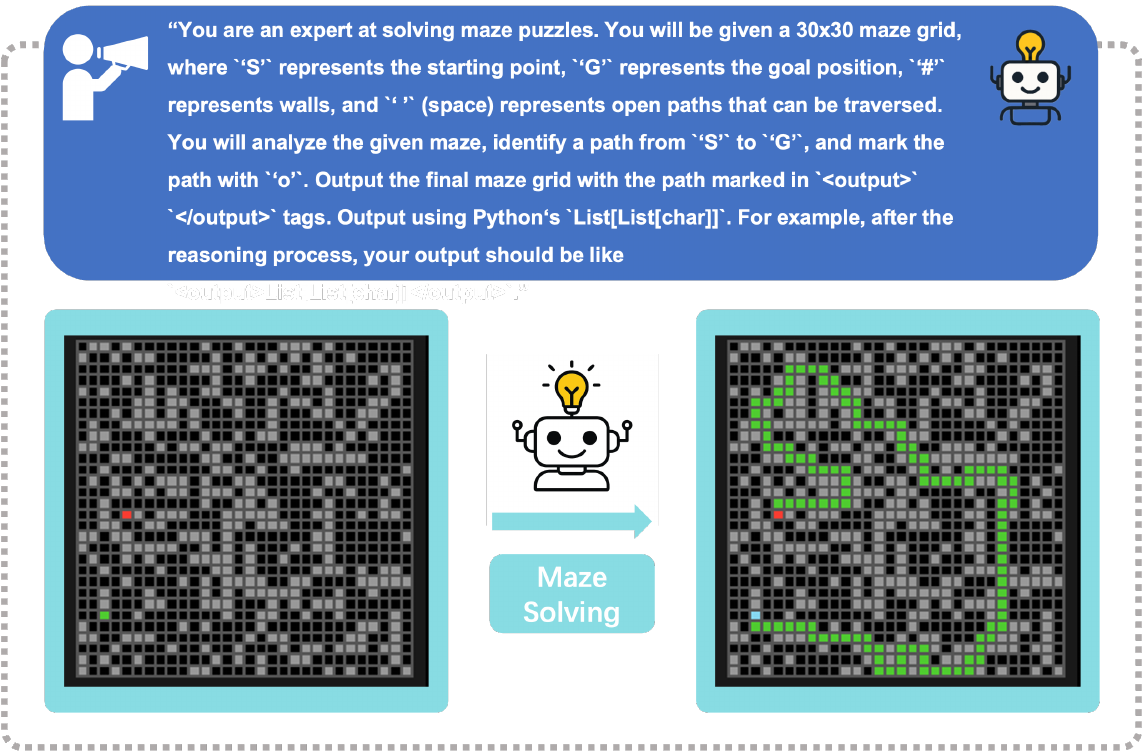}
    \caption{Maze-solving task that is constructed as a 30×30 symbolic grid. The left panel shows the raw input grid, and the right panel shows the model’s predicted solution path.}
    \label{fig:appendix_maze}
\end{figure*}

\begin{figure*}[hb!]
    \centering
    \includegraphics[width=\linewidth]{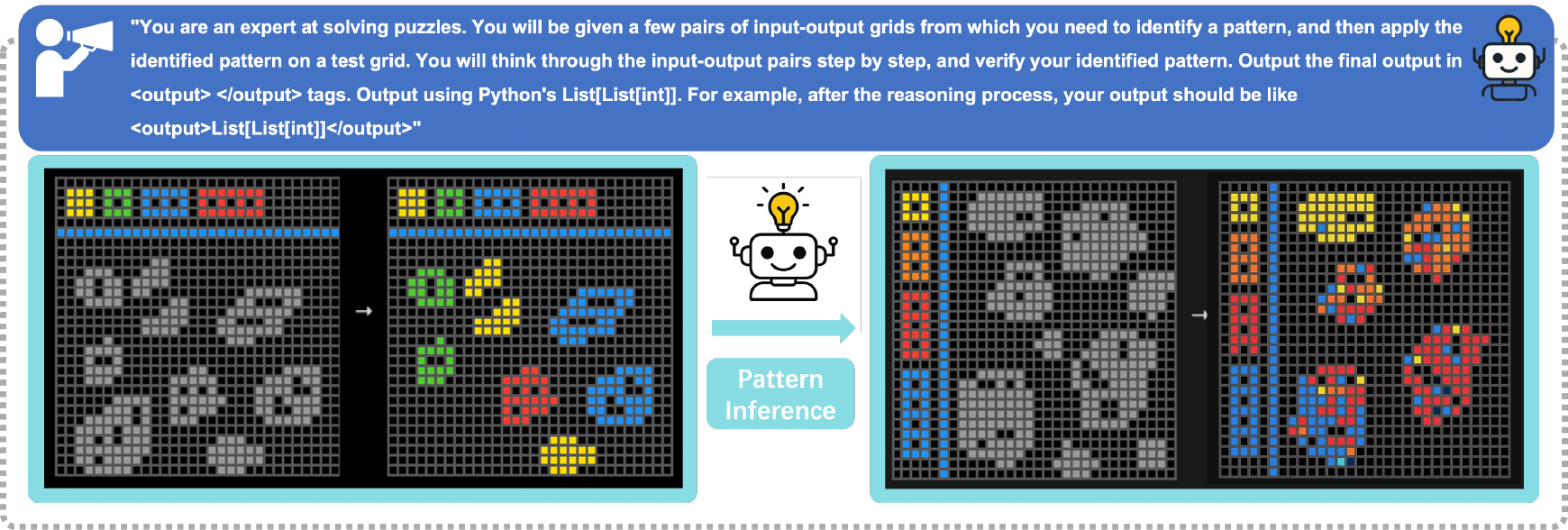}
    \caption{ARC-style pattern induction task. The model receives pairs of input–output grids, then must infer the underlying transformation and apply it to a new test input grid.}
    \label{fig:appendix_arc}
\end{figure*}

\end{document}